\def\year{2019}\relax
\documentclass[letterpaper]{article} 
\usepackage{aaai19}  
\usepackage{times}  
\usepackage{helvet}  
\usepackage{courier}  
\usepackage{url}  
\usepackage{graphicx}  
\usepackage{amssymb}
\usepackage{amsthm}
\usepackage{amsmath} 
\usepackage{subfiles}

\DeclareMathOperator*{\argmin}{argmin}
\newcommand{\PreserveBackslash}[1]{\let\temp=\\#1\let\\=\temp}
\newcommand{\norm}[1]{\lVert#1\rVert}

\usepackage{graphicx}
\usepackage{subfig}
\usepackage[linesnumbered, algoruled]{algorithm2e}
\usepackage{multirow}
\usepackage{array}
\usepackage{threeparttable}
\usepackage{tabularx, booktabs}
\newtheorem{thm}{Theorem}
\newtheorem{defi}{Definition}

\newtheorem{prop}{Proposition}
\newtheorem{proper}{Property}
\newtheorem{lem}{Lemma}
\newtheorem{rem}{Remark}

\newcommand{\regg}[1 ] {\sum_{i=\kappa+1}^{min\{d,U\}} \sigma_i^2(\boldsymbol{G}^{#1})}

\def \fk#1 {\mathcal{F}(\boldsymbol{\theta}^{#1},\boldsymbol{G}^{#1}, \boldsymbol{P}^{#1} )}
\def \lk#1 {\mathcal{L}(\boldsymbol{\theta}^{#1},\boldsymbol{G}^{#1}, \boldsymbol{P}^{#1})}
\def \thk {\left<\nabla_{\boldsymbol{\theta}}\mathcal{L}, \Delta(\boldsymbol{\theta}^k)\right> + \frac{\rho_{k+1}}{2}\norm{\Delta(\boldsymbol{\theta}^k)}^2_2}

\def \gk {\left<\nabla_{\boldsymbol{G}}\mathcal{L}, \Delta(\boldsymbol{G}^k)\right> + \frac{\rho_{k+1}}{2}\norm{\Delta(\boldsymbol{G}^k)}^2_F}

\def \pk {\left<\nabla_{\boldsymbol{P}}\mathcal{L}, \Delta(\boldsymbol{P}^k)\right> + \frac{\rho_{k+1}}{2}\norm{\Delta(\boldsymbol{P}^k)}^2_F}

\def \rth{\thk + \lambda_1 \norm{\boldsymbol{\theta}^{k+1}}^2_2}
\def \rgg{\gk + \lambda_2 \regg{k+1}}
\def \rp{\pk + \lambda_3 \norm{\boldsymbol{P}^{k+1}}_{1,2}}

\title{ Learning Personalized Attribute Preference via  Multi-task AUC Optimization }
\author
{
	Zhiyong Yang,\textsuperscript{1,2} Qianqian Xu,\textsuperscript{3}  Xiaochun Cao,\textsuperscript{1,2} Qingming Huang \textsuperscript{3,4,5} \thanks{The corresponding author.}\\
	\textsuperscript{1}{SKLOIS, Institute of Information Engineering, Chinese Academy of Sciences, Beijing, China}\\
	\textsuperscript{2}{School of Cyber Security, University of Chinese Academy of Sciences, Beijing, China}\\
	\textsuperscript{3}{Key Lab of Intell. Info. Process., Inst. of Comput. Tech., CAS, Beijing, China}\\
	\textsuperscript{4}{University of Chinese Academy of Sciences, Beijing, China}\\ \textsuperscript{5}{Key Laboratory of Big Data Mining and Knowledge Management, CAS, Beijing, China}\\
	\{yangzhiyong, caoxiaochun\}@iie.ac.cn, xuqianqian@ict.ac.cn, qmhuang@ucas.ac.cn	\\
}
\begin{document}
	
	\maketitle
	
	\begin{abstract}
		Traditionally, most of the existing attribute learning methods are trained based on the consensus of annotations aggregated from a limited number of annotators. However, the consensus might fail in settings, especially when a wide spectrum of annotators with different interests and  comprehension about the attribute words are involved. In this paper, we develop a novel multi-task method to understand and predict personalized attribute annotations. Regarding the attribute preference learning for each annotator as a specific task, we first propose a multi-level task parameter decomposition to capture the evolution from a highly popular opinion of the mass to highly personalized choices that are special for each person. Meanwhile, for personalized learning methods, ranking prediction is much more important than accurate classification. This motivates us to employ an Area Under ROC Curve (AUC) based loss function to improve our model. On top of the AUC-based loss, we propose an efficient method to evaluate the loss and gradients. Theoretically, we propose a novel closed-form solution for one of our non-convex subproblem, which leads to provable convergence behaviors. Furthermore, we also provide a generalization bound to guarantee a reasonable performance.  Finally, empirical analysis consistently speaks to the efficacy of our proposed method.
	\end{abstract}
	\section{Introduction}
	Visual attributes are semantic cues describing visual properties such as texture, color, mood, and \textit{etc}. Typical instances are  \textit{comfortable} or \textit{high heeled} for shoes, and \textit{smiling} or \textit{crying} for human faces. During the past decade, attribute learning has emerged as a powerful building block for a wide range of applications \cite{face1,reid2,zeroshot2,attrper}. \\
	\indent The status quo of the attribute learning methods are mostly based on the global labels aggregated from few annotators \cite{attrexp2,attrexp3,attrexp1}. Recently, the rise of online crowd-sourcing platforms (like Amazon Mechanical Turk) makes collecting attribute annotations from a  broad variety of annotators possible \cite{user2}, which offers us a chance to revisit the attribute learning. For consensus attribute learning, the underlying assumption is that the user decisions perturb slightly at random around the common opinion. However, different annotators might very well have distinct comprehension regarding the meaning of the attributes \textit{(typical ones like "open", "fashionable")}. This suggests that the gap between  personalized decisions and  common opinions could not be simply interpreted  as random noises. What's worse, one might even come to find conflicting results from different users. In such a case, there is a need to learn the consensus effects as well as personalized effects simultaneously, especially when personalized annotations are available. There are two crucial issues that should be noticed in this problem.\\
	\indent The grouping effects, lying in between consensus and personalized ones, also play an essential role in understanding the user-specific attribute annotations. As pointed out in the previous literature \cite{user2}, when understanding semantic attributes, humans often form the "school of thoughts" in terms of their cultural backgrounds and the way they interpret the semantic words. Though personalized effect might lead to conflicting results from different schools or groups, the users within the same group might very likely provide similar decisions. Moreover, different groups of people may probably favor distinct visual cues, as there is significant diversity among user groups. In other words, each user group should be assigned with a distinct feature subset. Accordingly, visual features and users should be simultaneously grouped to guarantee better performance. Furthermore, seeing that we cannot obtain the user-feature groups in advance, this constraint should be implicitly and automatically reflected via the structure of  model parameters.\\
	\indent  Unlike the consensus-based attribute predictions, preference learning (e.g. recommendation and image searching)  is more important than label prediction for personalized attribute learning. Under such circumstances, when the attribute words are used as keywords or tags, it should be guaranteed that the positive labeled instances are ranked higher than the negative ones. It is well-known that the Area Under the ROC curve (AUC) metric exactly meets this requirement \cite{AUCconcept}, and thus a better objective for our task.\\
	\indent  Our goal in this paper is to learn personalized attributes based on the two mentioned issues. More precisely, we regard attribute preference learning for a specific user as a task. On top of this, we propose a multi-task model for the problems we tackle. Our main contributions are listed as follows: a) In the multi-task model, we propose a three-level decomposition of the task parameters which includes a consensus factor, a user-feature co-clustering factor and a personalized factor. b) The proximal gradient descent method is adopted to solve the model parameters. Regarding our contribution here, we derive a novel closed-form solution for the proximal operator of the co-clustering factor and provide an efficient AUC-based evaluation method. c) Systematic theoretical analyses are carried out on the convergence behaviors and generalization bounds of our proposed method, while empirical studies are carried out for a simulation dataset and two real-world attribute annotation datasets. Both theoretical and empirical results suggest the superiority of our proposed algorithm. \textit{\textbf{The codes are now available.}} \footnote{https://joshuaas.github.io/publication.html} 
	\section{Related Work}
	\subsubsection{Attribute learning}Attribute learning has long been playing a central role in many machine learning and computer vision problems. Along this line of research, there are some previous studies that investigate the personalization of attribute learning. \cite{user1} learns user-specific attributes with  an adaption process. More precisely, a general model is first trained based on a large pool of data. Then a small user-specific dataset is employed to adapt the trained model to user-specific predictors. \cite{user2} argues that one attribute might have different interpretations for different groups of persons. Correspondingly, a shade discovery method is proposed therein to leverage group-wise user-specific attributes. Both works adopt two-stage or multi-stage models and even extra dataset. For the group modeling, \cite{user2} only focuses on user-level grouping while ignores the grouping effect of the features. Furthermore, the merit of AUC is also neglected. In contrast, we propose a fully automatic AUC-based attribute preference learning model where the user-feature coclustering effect is considered.
	\subsubsection{Multi-task Learning}
	Multi-task learning  aims at improving the generalization
	performance by sharing information among multiple tasks. Many efforts have been made to improve multi-task learning \cite{RMTL,rMTFL,RAMUSA,MTATT,comitlasso}, etc. Recently, there is also a wave to investigate the clustering and grouping based multi-task learning \cite{CMTL,GOMTL,cocluster}. Among this works, \cite{cocluster} is the most relevant work compared with our work since we adopt the co-clustering regularization proposed therein. However, our work differs significantly in that 1) our model is specially designed for the attribute preference learning problem; 2) we propose a novel hierarchical decomposition scheme for the model parameters and; 3) we propose a novel closed-form solution for the proximal mapping of the co-clustering penalty; 4) we focus on efficient AUC optimization instead of regression or classification.
	\section{Model Formulation}
	In this section, we propose a novel AUC-based multi-task model. Specifically, we first introduce the notations used in this paper, followed by problem settings in our model and a multi-level parameter decomposition. After that, we systematically elaborate two building blocks of our model: the AUC-based loss and evaluation, and the regularization scheme.
	\subsection{Notations}
	\indent $\left<\cdot,\cdot\right>$ denotes the inner product for two matrices or two vectors. The singular values of a matrix $\boldsymbol{A}$ are denoted as $\sigma_1(\boldsymbol{A}), \cdots,\sigma_m(\boldsymbol{A})$ such that $\sigma_1(\boldsymbol{A}) \ge \sigma_2(\boldsymbol{A})\ge \cdots, \ge \sigma_m(\boldsymbol{A}) \ge 0$. $\boldsymbol{I}$ is the identity matrix, 
	$I_{[\mathcal{A}]}$ is the indicator function of the set $\mathcal{A}$, $\boldsymbol{1}$ denotes the all-one vector or matrix. $\mathcal{U}(a,b)$ denotes the uniform distribution and $\mathcal{N}(\mu, \sigma^2)$ denotes the normal distribution. $\otimes$ denotes the Cartesian product.
	
	\subsection{Problem Settings}
	\begin{figure}[ht]
		\centering
		\subfloat{
			\includegraphics[width =0.85\columnwidth]{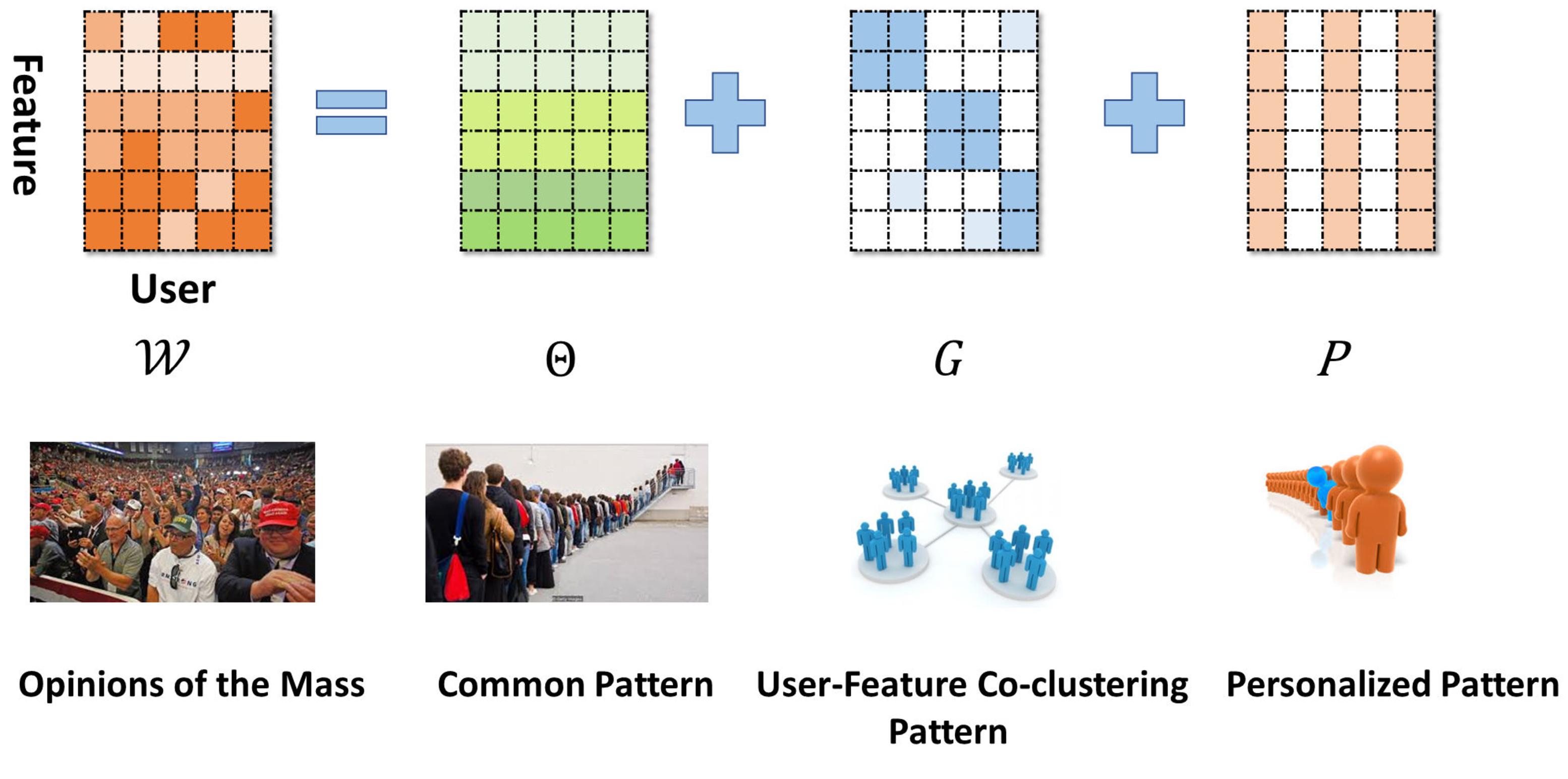}
		}
		\caption{An illustration of the Multi-level Decomposition of the model parameters.}
		\label{fig:decomp}
	\end{figure}
	\indent For a given attribute\footnote{In our model, different attributes are learned separately. In the rest of this paper, the discussions  focus on a specific attribute. }, assume that we have $U$ users who have annotated a given set of images. Further, we assume that the $i$th user labeled $n_i$ images with $n_{+,i}$ positive labels and $n_{-,i}$ negative labels. $\mathcal{S}_{+,i}=
	\{k \ |  \ y^{(i)}_k = 1\}$  and $\mathcal{S}_{-,i}=
	\{k \ |  \ y^{(i)}_k = -1\}$. We denote the training data as: $\mathcal{S} = \left\{(\boldsymbol{X}^{(1)}, \boldsymbol{y}^{(1)}), \cdots, (\boldsymbol{X}^{(U)}, \boldsymbol{y}^{(U)})\right\}$. For $\mathcal{S}$, $\boldsymbol{X}^{(i)} \in \mathbb{R}^{n_{i} \times d}$
	is the image feature inputs for the images that the $i$th user labeled. Each row of $\boldsymbol{X}^{(i)}$ represents the extracted features for a corresponding image.   $\boldsymbol{y^{(i)}} \in \{-1,1\}^{n_{i}}$ is the corresponding label vector. If $y^{(i)}_k =1 $, then the user thinks that the $k$th image bears the given attribute, otherwise we have $y^{(i)}_k =-1 $.\\
	\indent Taking advantage of the multi-task learning paradigm, the attribute preference learning for a user is regarded as a specific task.  Our goal is then to learn all the task models $\boldsymbol{f}^{(i)}(x)$ , where, for each task, a linear model is learned as the scoring function, i.e.,  $\boldsymbol{f}^{(i)}(\boldsymbol{x})= \boldsymbol{W}^{(i)^\top}\boldsymbol{x}$.\\
	\indent  As shown in the introduction,  it is natural to observe diversity in personalized scores. However, this diversity could by no means goes arbitrary large. In fact, we could interpret such  limited diversity in a consensus-to-personalization manner. A common pattern is shared among the mass that captures the popular opinion. Different people might have different bias and preference, which drives them away from a consensus. Users sharing similar biases tend to form groups. The users within a  group  share similar biases towards the popular opinion based on a similar subset of the features of the object. Finally, a highly personalized user in a group
	tends to adopt an extra bias toward the group opinion. Mathematically, this interpretation induces a multi-level decomposition of the model weights :
	$\boldsymbol{W}^{(i)} = \boldsymbol{\theta} + \boldsymbol{G}^{(i)} + \boldsymbol{P}^{(i)}. $
	$\boldsymbol{\theta} \in \mathbb{R}^{d\times1}$ is the  common factor that captures the popular global preference. $\boldsymbol{G}^{(i)} \in \mathbb{R}^{d\times1} $ is the grouping factor for the $i$-th task.  For mathematical convenience, we denote $\boldsymbol{G}=[\boldsymbol{G}^{(1)},\cdots, \boldsymbol{G}^{(U)}]$, and we have $\boldsymbol{G}\in \mathbb{R}^{d\times U }$.
	$\boldsymbol{P}^{(i)}$ is the user-specific factor mentioned above. Similarly, we define $\boldsymbol{P} = [\boldsymbol{P}^{(1)},\cdots,\boldsymbol{P}^{(U)}]$, and $\boldsymbol{P} \in \mathbb{R}^{d \times U}$. An illustration of this decomposition is shown in Figure \ref{fig:decomp}.
	
	With all the above-mentioned settings, we adopt a general objective function in the form:
	\begin{equation}\label{genform}
	\min_{\boldsymbol{\theta}, \boldsymbol{G}, \boldsymbol{P}} \sum\limits_{i=1}^{U}\ell_i(\boldsymbol{f}^{(i)}, \boldsymbol{y}^{(i)}) + \lambda_1\mathcal{R}_1(\boldsymbol{\theta}) + \lambda_2 \mathcal{R}_2(\boldsymbol{G}) + \lambda_3\mathcal{R}_3(\boldsymbol{P}).
	\end{equation}
	Given (\ref{genform}), there are two crucial building blocks to be determined:
	\begin{itemize}
		\item The empirical loss function for a specific user $i$ : $\ell_i(\cdot, \cdot)$ which directly induces AUC optimization; 
		\item Regularization terms  $\mathcal{R}_1(\boldsymbol{\theta})$,  $\mathcal{R}_2(\boldsymbol{G})$, and  $\mathcal{R}_3(\boldsymbol{P})$ which are defined by prior constraints on $\boldsymbol{W}$.
	\end{itemize}
	In what follows, we will elaborate the formulation of two building blocks, respectively. 
	\subsection{Regularization}
	\indent For the common factor $\boldsymbol{\theta}$, we simply adopt the most widely-used $\ell_2$ regularization $\mathcal{R}_1(\boldsymbol{\theta}) = \norm{\boldsymbol{\theta}}_2^2$ to reduce the model complexity. For $\boldsymbol{G}$, as mentioned in the previous parts, what we pursue here is a user-feature co-clustering effect. A previous work in \cite{cocluster}  shows that one way to simultaneously cluster the rows and columns of a matrix in $\boldsymbol{R}^{m \times n}$ into $\kappa$ groups is to penalize  the sum of squares of the bottom $min\{n,m\} - \kappa$ singular values. This motivates us to adopt a regularizer on $\boldsymbol{G}$ in the following form :$
	\mathcal{R}_2(\boldsymbol{G}) = \sum_{\kappa+1}^{min\{d,U\}}\sigma_i^2(\boldsymbol{G}).$
	For any user $i$, a non-zero column $\boldsymbol{P}^{(i)}$ is favorable only when she/he has  a significant disagreement with the common-level and the group-level results. This inspires us to define $\mathcal{R}_3(\boldsymbol{P}) =\norm{\boldsymbol{P}}_{1,2}$ norm to induce column-wise sparsity. 
	\subsection{Empirical Loss and Its Evaluation}
	Since the empirical loss is evaluated separately for each user, without loss of generality,  the following discussion only focuses on a given user $i$. 
	\subsubsection{Empirical Loss} AUC is defined as the probability that a randomly sampled positive instance has a higher predicted score than a randomly sampled negative instance. Since we need to minimize our objective function, we focus on the loss version of AUC, i.e., the mis-ranking probability.
	Though the data distribution is unknown, given each user $u_i$ and $\mathcal{S}_{+,i}$, $\mathcal{S}_{-,i}$ defined in the problem setting, we could attain a finite sample-based estimation of the loss version of AUC: \[ \ell^{(i)}_{AUC} = \sum\limits_{x_p \in \mathcal{S}_{+,i}}\sum\limits_{x_q \in \mathcal{S}_{-,i}} \frac{I(\boldsymbol{x}_p, \boldsymbol{x}_q)}{n_{+,i}n_{-,i}},\]
	where $I(\boldsymbol{x}_p, \boldsymbol{x}_q)$ is a discrete mis-ranking punishment in the form: $ I(\boldsymbol{x}_p, \boldsymbol{x}_q) =  {I}_{[f^{(i)}(\boldsymbol{x}_p) > f^{(i)}(\boldsymbol{x}_q)]} + \frac{1}{2}~ {I}_{[f^{(i)}(\boldsymbol{x}_p) = f^{(i)}(\boldsymbol{x}_q)]}. $
	It is easy to see that $\ell^{(i)}_{AUC}$ is exactly the mis-ranking frequency for user $i$ on the given dataset. Unfortunately, optimizing this metric directly is an $\mathcal{NP}$ hard problem. To address this issue, we adopt the squared surrogate loss $s(t)= (1-t)^2$ \cite{AUC2}. Accordingly, the empirical loss $\ell_i(\boldsymbol{f}^{(i)}, \boldsymbol{y}^{(i)})$ could be defined as: $\ell_i(\boldsymbol{f}^{(i)}, \boldsymbol{y}^{(i)}) = \sum\limits_{x_p \in \mathcal{S}_{+,i}}\sum\limits_{x_q \in \mathcal{S}_{-,i}} \frac{s\Big(\boldsymbol{f}^{(i)}(\boldsymbol{x}_p) -\boldsymbol{f}^{(i)}(\boldsymbol{x}_q) \Big)}{n_{+,i}n_{-,i}}
	$.
	\begin{figure}[ht]
		\centering
		\subfloat{
			\includegraphics[width =0.85\columnwidth]{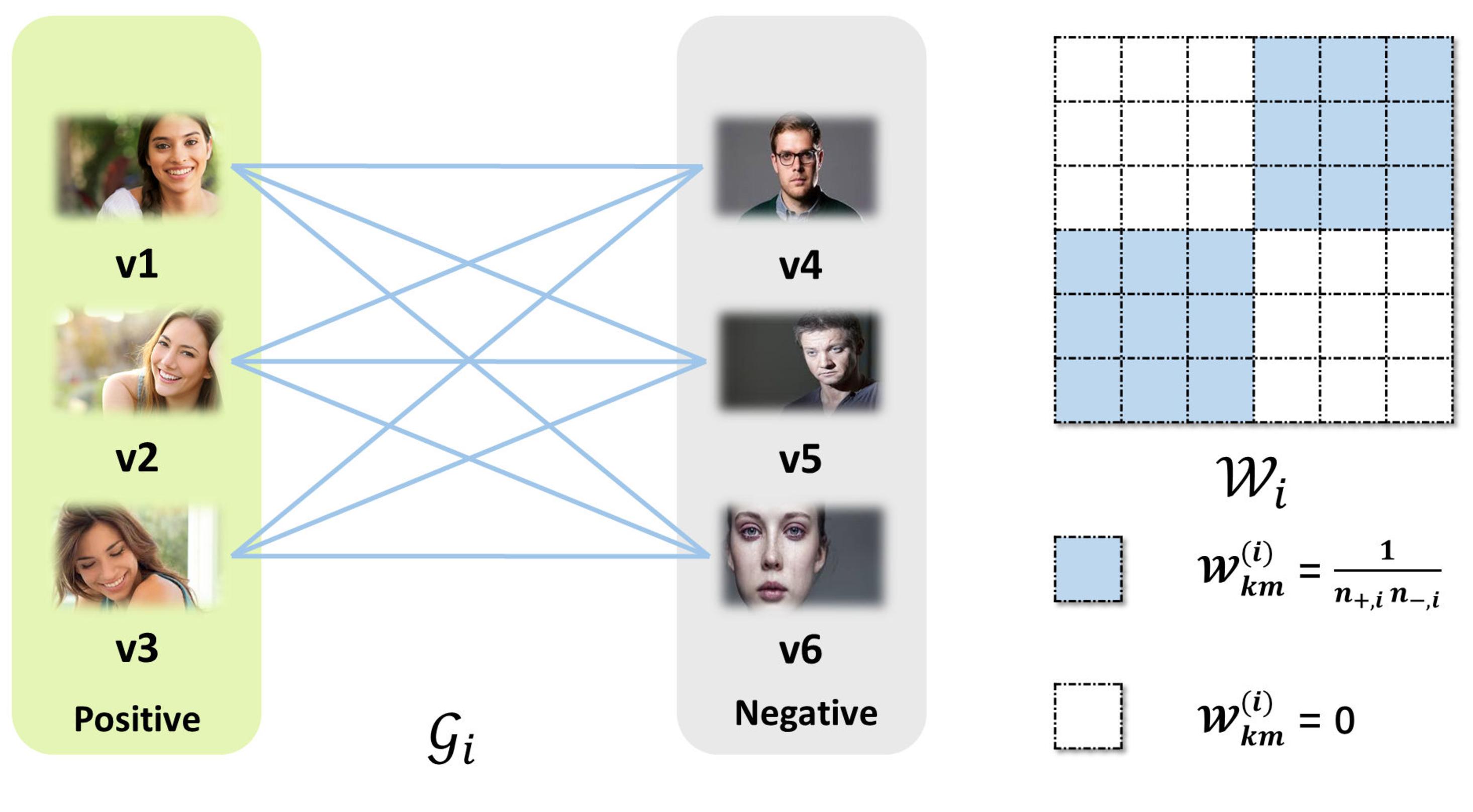}
		}
		\caption{An illustration of the AUC graph, taking the annotation for attribute \textit{smile} as an example.}
		\label{fig:auc}
	\end{figure}
	\subsubsection{Efficient AUC-based Evaluation} At the first glance, the pair-wise AUC loss induces much heavier computation burdens than the instance-wise losses. It is interesting  to note that, after carefully reformulating $\ell_i$, the computational burden coming from the pair-wise formulation could be perfectly eliminated. To see this, let us define a graph as $\mathcal{G}^{(i)} = (\mathcal{V}^{(i)}, \mathcal{E}^{(i)},\mathcal{W}^{(i)})$. The vertex set $\mathcal{V}^{(i)}$ is the set of all the instances in $(\boldsymbol{X}^{(i)},\boldsymbol{y}^{(i)})$.  There exists an edge $(k,m) \in \mathcal{E}^{(i)}$ with weight $\mathcal{W}^{(i)}_{km} = \frac{1}{n_{+,i}n_{-,i}}$ if and only if $y^{(i)}_k \ne y^{(i)}_m$. This graph is further illustrated in Figure \ref{fig:auc}. Given $\mathcal{W}^{(i)}$,  the Laplacian matrix  $\boldsymbol{L}^{(i)}$ of $\mathcal{G}_i$ 
	could be expressed as: $\boldsymbol{L}^{(i)} = diag(\mathcal{W}^{(i)}\boldsymbol{1}) - \mathcal{W}^{(i)}.$
	The empirical loss could be reformulated as a quadratic form defined by $\boldsymbol{L}^{(i)}$:
	$
	\ell_i(\boldsymbol{f}^{(i)}, \boldsymbol{y}^{(i)}) = \frac{1}{2} (\boldsymbol{\tilde{y}^{(i)}}-\boldsymbol{f}^{(i)})^\top\boldsymbol{L}^{(i)}(\boldsymbol{\tilde{y}^{(i)}}-\boldsymbol{f}^{(i)}),
	$
	where $\boldsymbol{\tilde{y}^{(i)}} = \frac{\boldsymbol{y}^{(i)}+1}{2} $. We see the AUC loss evaluation involves computing a quadratic form of $\boldsymbol{L}^{(i)}$. The following proposition gives a general result which  suggests an efficient method to compute $\boldsymbol{A}^\top\boldsymbol{L}^{(i)}\boldsymbol{B}$, and $\boldsymbol{A}^\top \boldsymbol{L}^{(i)}$. 
	
	\begin{prop}
		For any $\boldsymbol{A} \in \mathbb{R}^{n_i \times a}$ and $\boldsymbol{B} \in \mathbb{R}^{n_i \times b}$ ,where $a$ and $b$ are positive integers.  $\boldsymbol{A}^\top\boldsymbol{L}^{(i)}\boldsymbol{B}$, and $\boldsymbol{A}^\top\boldsymbol{L}^{(i)}$ could be finished within $\mathcal{O}\big(n_i(a+b+ ab)\big) =\mathcal{O}(abn_i)$ and $\mathcal{O}(an_i)$, respectively.
	\end{prop}
	\begin{rem}
		According to this proposition, the complexity of $\boldsymbol{A}^\top\boldsymbol{L}^{(i)}\boldsymbol{B}$
		could be reduced from $O(abn_{+,i}n_{-,i})$ to $O(ab(n_{+,i} + n_{-,i}))$, whereas the complexity of $\boldsymbol{A}^\top\boldsymbol{L}$ could be reduced from $O(an_{+,i}n_{-,i})$ to $O(a(n_{+,i} + n_{-,i}))$.
	\end{rem}

	\indent To end this section, we summarize our final objective function as: 
	\begin{equation*}
	\begin{split}
	(P^*)\min_{\boldsymbol{\theta}, \boldsymbol{G}, \boldsymbol{P}} ~ &\underbrace{\sum_{i}~\sum\limits_{\boldsymbol{x}_p \in \mathcal{S}_{+,i} } \sum\limits_{\boldsymbol{x}_q \in \mathcal{S}_{-,i} }\frac{s\Big(\boldsymbol{W}^{(i)^\top }(\boldsymbol{x}_p-\boldsymbol{x}_q)  \Big)}{n_{+,i}n_{-,i}}}_{\mathcal{L}(\boldsymbol{W})}\\ &+\lambda_1 \underbrace{\norm{\boldsymbol{\theta}}^2_{2}}_{\mathcal{R}_1(\boldsymbol{\theta})}
	+ \lambda_2 \underbrace{\sum_{\kappa+1}^{min\{d,U\}}\sigma^2_i(\boldsymbol{G})}_{\mathcal{R}_2(\boldsymbol{G})}
	+ \lambda_3 \underbrace{\norm{\boldsymbol{P}}_{1,2}}_{\mathcal{R}_3(\boldsymbol{P})}\\
	s.t ~ ~ ~ ~ ~ &\boldsymbol{W}^{(i)} = \boldsymbol{\theta} + \boldsymbol{G}^{(i)} + \boldsymbol{P}^{(i)}
	\end{split}
	\end{equation*}
	For the sake of simplicity, we denote the empirical loss as $\mathcal{L}(\boldsymbol{W})$,
	and we denote the objective function of ($P^*$) as $\mathcal{F}(\boldsymbol{\theta}, \boldsymbol{G},\boldsymbol{P})$.
	Note that $\mathcal{L}(\boldsymbol{W})$ should be a function of $\boldsymbol{\theta}$, $\boldsymbol{G}$, and $\boldsymbol{P}$.  
	\section{Optimization}
	We adopt  the proximal gradient method as the optimizer for our problem. In this section, we introduce the outline of the optimization method and provide a novel closed-form solution for the proximal operator of $\mathcal{{R}}_2(\boldsymbol{G})$.   \\
	\indent For each iteration step $k$, giving a reference point  $ \boldsymbol{W}^{ref_k} = (\boldsymbol{\theta}^{ref_k}, \boldsymbol{G}^{ref_k}, \boldsymbol{P}^{ref_k})$, then the proximal gradient method updates the variables as :
	\begin{align}
	& \boldsymbol{\theta}^k  ~:= \argmin_{\boldsymbol{\theta}} \dfrac{1}{2} \left\norm{\boldsymbol{\theta} -\tilde{\boldsymbol{\theta}^k} \right}_2^2 + \frac{\lambda_1}{\rho_k} \norm{\boldsymbol{\theta}}^2_2\label{Ptheta} \\
	&\boldsymbol{G}^{k}  := \argmin_{\boldsymbol{G}} \dfrac{1}{2}\left\norm{\boldsymbol{G} - \tilde{\boldsymbol{G}}^{k}\right}_F^2 + \frac{\lambda_2}{\rho_k} \sum_{\kappa+1}^{min\{d,U\}}\sigma^2_i(\boldsymbol{G})\label{Pg}\\
	&\boldsymbol{P}^k  := \argmin_{\boldsymbol{P}} \dfrac{1}{2}\left\norm{\boldsymbol{P} - \tilde{\boldsymbol{P}}^k\right}_F^2 + \frac{\lambda_3}{\rho_k} \norm{\boldsymbol{P}}_{1,2}\label{Pp}
	\end{align}	
	where: 
	$\tilde{\boldsymbol{\theta}}^k = \boldsymbol{\theta}^{ref_k} - \dfrac{1}{\rho_k}\nabla_{\theta}\mathcal{L}(\boldsymbol{W}^{ref_k})$,
	$\tilde{\boldsymbol{G}}^{k}=\boldsymbol{G}^{ref_k} - \dfrac{1}{\rho_k}\nabla_{\boldsymbol{G}}\mathcal{L}(\boldsymbol{W}^{ref_k})$,
	$\tilde{\boldsymbol{P}}^k=\boldsymbol{P}^{ref_k} - \dfrac{1}{\rho_k}\nabla_{P}\mathcal{L}(\boldsymbol{W}^{ref_k})$
	
	\noindent and $\rho_k$ is chosen with a line-search strategy, where we keep on updating $\rho_k =\alpha\rho_k, \alpha>1$ until it satisfies:
	\begin{equation}\label{upper bound}
	\mathcal{L}(\boldsymbol{W}) <
	\mathcal{L}(\boldsymbol{W}^{ref_k}) + \Psi_{\rho_k}(D\boldsymbol{\theta})
	+  \Psi_{\rho_k}(D\boldsymbol{G}) +  \Psi_{\rho_k}(D\boldsymbol{P}).
	\end{equation}
	Here $D \boldsymbol{\theta} = \boldsymbol{\theta} - \boldsymbol{\theta}^{ref_k}$,$D \boldsymbol{G}= \boldsymbol{G} - \boldsymbol{G}^{ref_k}$ and $D \boldsymbol{P} = \boldsymbol{P} - \boldsymbol{P}^{ref_k}$, $\Psi_{\rho_k}(D{A}) =\left<\nabla_{\boldsymbol{A}} \mathcal{L}(\boldsymbol{W}^{ref_k}),  D \boldsymbol{A} \right> + \frac{\rho_k}{2} \left< D \boldsymbol{A}, D \boldsymbol{A} \right> $. 
	\begin{rem}
		The existence of such $\rho_k$ is guaranteed by the Lipschitz continuity of $\nabla\mathcal{L}(\boldsymbol{W})$. Note that, we choose the last historical update of the parameter as the reference point i.e. $\boldsymbol{W}^{ref{_k}} = \boldsymbol{W}^{k-1}$. 
	\end{rem} 
	\indent The solution to Eq.(\ref{Ptheta}) and Eq.(\ref{Pp}) directly follows the proximal operator of the $\ell_2$ norm and the $\ell_{1,2}$ norm \cite{optml}. For Eq.(\ref{Pg}), we provide a novel closed-form solution in the following subsection. 
	\subsubsection{A closed-form solution for $\boldsymbol{G}$ subproblem}
	Note that $\sum_{\kappa+1}^{min\{d,U\}}\sigma_i(\boldsymbol{G})^2$ is not convex, solving the $\boldsymbol{G}$ subproblem is challenging. Conventionally, this problem is  solved  in an alternative manner \cite{cocluster} which is inefficient and lack of theoretical guarantees. Thanks to the general singular value thresholding framework \cite{gsvt,trunWNN}, we could obtain a closed-form optimal solution according to the following proposition.
	\begin{prop}
		An Optimal Solution of (\ref{Pg}) is:
		\begin{equation} \label{solveP}
		\boldsymbol{G}^* = \boldsymbol{U}\mathcal{T}_{\kappa,\frac{\lambda_3}{\rho_k}}(\boldsymbol{\Sigma})\boldsymbol{V}^\top
		\end{equation}
		where $\boldsymbol{U}\boldsymbol{\Sigma}\boldsymbol{V}^\top$ is a SVD decomposition of $\tilde{\boldsymbol{G}}^{k}$, $\mathcal{T}_{\kappa,c}$ maps $\boldsymbol{\Sigma} =diag(\sigma_1,\cdots,\cdots,\sigma_{min\{d,U\}})$ to a diagonal matrix having the same size with $\mathcal{T}_{\kappa,c}(\boldsymbol{\Sigma})_{ii} = (\frac{1}{2c+1})^{I[i>\kappa]}\sigma_i$.
	\end{prop}
	
	%
	%
	%
	\section{Theoretical Analysis}
	\textit{Due to the space limit, all the details of the proofs are released at a Github homepage\footnote{https://joshuaas.github.io/publication}.}
	\subsection{Lipschitz Continuity of the Gradients of $\mathcal{L}(\boldsymbol{W})$}
	In the preceding section, we have pointed out that the Lipschitz Continuity of the Gradients of $\mathcal{L}(\boldsymbol{W})$ is a necessary condition for the success of the line search process to find $\rho_k$. Now in the following theorem, we formally prove this property as theoretical support for the optimization method.
	\begin{thm}[Lipschitz Continuous Gradient]
		Suppose that the data is bounded in the sense that:
		\[\forall i, ~\norm{\boldsymbol{X}^{(i)}}_2 =\sigma_{X_i} < \infty, ~n_{+,i} \ge 1, ~n_{-,i} \ge 1.\]
		Given two arbitrary distinct  parameters $\boldsymbol{W}, \boldsymbol{W}'$, we have:
		\begin{equation*}
		\norm{\nabla{\mathcal{L}(vec(\boldsymbol{W}))} - \nabla{\mathcal{L}(vec(\boldsymbol{W}'))}} \leq \gamma \Delta \boldsymbol{W}
		\end{equation*}
		where: $\gamma 
		=3U\sqrt{(2U+1)}\max_{i} \left\{\dfrac{n_i\sigma^2_{X_i}}{n_{+,i}n_{-,i}}\right\}
		$, $vec(\boldsymbol{W})= [\boldsymbol{\theta}, vec(\boldsymbol{G}), vec(\boldsymbol{P})]$, $\Delta \boldsymbol{W} = \norm{vec(\boldsymbol{W}) - vec(\boldsymbol{W}') }$.
	\end{thm}
	\subsection{Convergence Analysis}
	\indent Since  the regularization term $\sum_{i=\kappa+1}^{min\{d,U\}} \sigma_i^2(\boldsymbol{G})$ is non-convex, the traditional sub-differential  is not fully available anymore. In this paper, we adopt the generalized  sub-differential defined in \cite{var,liu}. To guarantee a nonempty sub-differential set, the objective function must be lower semi-continuous. In our problem, it is obvious that $\mathcal{L}(\boldsymbol{W})$, $\mathcal{R}_1(\boldsymbol{\theta})$, $\mathcal{R}_2(\boldsymbol{P})$ are lower semi-continuous functions by their continuity. For the non-convex term  $\sum_{i=\kappa+1}^{min\{d,U\}} \sigma_i^2(\boldsymbol{G})$, the following lemma shows that is also continuous and thus lower semi-continuous. 
	\begin{lem}\label{conti}
		The function $\sum_{i=\kappa+1}^{min\{d,U\}} \sigma_i^2(\boldsymbol{G})$ is continuous with respect to $\boldsymbol{G}$.
	\end{lem}
	\noindent Then the convergence properties of the proposed method could be summarized in the following theorem. Here we define $\Delta(\boldsymbol{\theta}^k) = {\boldsymbol{\theta}^{k+1} - \boldsymbol{\theta}^k}, ~\Delta(\boldsymbol{G}^k) = {\boldsymbol{G}^{k+1} - \boldsymbol{G}^k}, \Delta(\boldsymbol{P}^k) = {\boldsymbol{P}^{k+1} - \boldsymbol{P}^k}. $
	
	\noindent\begin{thm} Assume that the initial solutions ${\boldsymbol{\theta}^0, \boldsymbol{G}^0, \boldsymbol{P}^0}$ are bounded, with the line-search strategy defined in (\ref{upper bound}), the following properties hold :
		\begin{itemize}
			\item 1) The sequence $\{ \fk{k} \}$ is non-increasing in the sense that : $\forall k, \exists C_{k+1} > 0,$
			\begin{equation}
			\begin{split}
			&\fk{k+1} \le \fk{k} - \\& C_{k+1}
			( \norm{\Delta(\boldsymbol{\theta}^k)}_2^2+ \norm{\Delta(\boldsymbol{G}^k)}_F^2 + \norm{\Delta(\boldsymbol{P}^k)}_F^2  )
			\end{split}
			\end{equation}
			
			\item 2) $\lim_{k \rightarrow \infty} \boldsymbol{\theta}^k - \boldsymbol{\theta}^{k+1}  = 0$, \ $\lim_{k \rightarrow \infty} \boldsymbol{G}^k - \boldsymbol{G}^{k+1}  = 0$,\\$\lim_{k \rightarrow \infty} \boldsymbol{P}^k - \boldsymbol{P}^{k+1}  = 0$.
			\item 3) The parameter sequences $\{\boldsymbol{\theta}^k\}_k$, $\{\boldsymbol{G}^k\}_k$, $\{\boldsymbol{P}^k\}_k$ are bounded
			\item 4) Every limit point of $\{\boldsymbol{\theta}^k, \boldsymbol{G}^k, \boldsymbol{P}^k\}_k$ is a critical point of the problem. 
			\item 5) $\forall ~T \ge 1, \exists~ C_{T} >0$ :
			\begin{align*}
			&\min_{0\le k < T}\left( \norm{\Delta(\boldsymbol{\theta}^k)}_2^2\right) \le  \frac{C_{T}}{T}, \ \min_{0\le k < T}\left( \norm{\Delta(\boldsymbol{G}^k)}_F^2\right) \le  \frac{C_{T}}{T}, \\
			&\min_{0\le k < T}\left( \norm{\Delta(\boldsymbol{P}^k)}_F^2\right) \le  \frac{C_{T}}{T}.
			\end{align*}
		\end{itemize}
	\end{thm}
	\subsection{Generalization Bound}
	Define the parameter set $\Theta$ as :
	\begin{equation*}
	\begin{split}
	&\Theta=\big\{(\boldsymbol{\theta},\boldsymbol{G},\boldsymbol{P}):\sqrt{\mathcal{R}_1({\boldsymbol{\theta}})} \le \psi_1, \mathcal{R}_2(\boldsymbol{G}) \le \psi_2, \\
	&\norm{\boldsymbol{G}}_2 \le \sigma_{max}< \infty,\mathcal{R}_3(\boldsymbol{P}) \le \psi_3  \big\}
	\end{split}
	\end{equation*}
	We have the following uniform bound.
	\begin{thm}
		Assume that $\exists \Delta_{\chi} > 0$,  all the instances are sampled such that, 
		$\norm{x} \le \Delta_{\chi} $ .Define $C=(\psi_1 +\sqrt{\psi_2 + \kappa \cdot \sigma_{max}^2 } + \psi_3)$ $\zeta$ as
		$ \zeta =\Delta_\chi C$,
		we have, for all $\delta \in(0,1)$, for all $(\boldsymbol{\theta},\boldsymbol{G},\boldsymbol{P}) \in \Theta$ :
		\begin{equation*}
		\begin{split}
		\mathbb{E}_\mathcal{D}(\sum_i\ell^{(i)}_{AUC}) \le &\mathcal{L}(\boldsymbol{W}) + \sum_{i=1}^U\frac{B_1}{\sqrt{(n_i\chi_i(1-\chi_i))}}\\
		&+ B_2\sqrt{\frac{ln(\frac{2}{\delta})}{\sum_{i=1}^Un_i\chi_i(1-\chi_i)}} 
		\end{split}
		\end{equation*}
		holds with probability at least $1-\delta$, where $B_1={8\sqrt{2}C\Delta_{\chi}(1+\zeta)}$, $B_2 = 	10\sqrt{2} (1+\zeta)\zeta$, $\chi_i = \frac{n_{+,i}}{n_i}$. The distribution $\mathcal{D} = \otimes_{i=1}^U (\mathcal{D}_{+,i} \otimes \mathcal{D}_{-,i})$, where for user $i$, $\mathcal{D}_{+,i}$, $\mathcal{D}_{-,i}$ are conditional distributions for positive and negative instances, respectively.
	\end{thm}
	\begin{rem}
		According to Theorem 2, the loss function is non-increasing. For the solution of our method $(\boldsymbol{\theta}^*,\boldsymbol{G}^*,\boldsymbol{P}^*)$, we then have:
		$\sqrt{\mathcal{R}_1({\boldsymbol{\theta^*}})} \le \sqrt{\frac{ \mathcal{F}(\boldsymbol{\theta}^{0},  \boldsymbol{G}^{0},\boldsymbol{P}^{0})   }  {\lambda1}},$
		$\mathcal{R}_2({\boldsymbol{G^*}}) \le \frac{\mathcal{F}(\boldsymbol{\theta}^{0}, \boldsymbol{G}^{0} ,\boldsymbol{P}^{0})  }{\lambda2}$,  $\mathcal{R}_3({\boldsymbol{P^*}}) \le \frac{\mathcal{F}(\boldsymbol{\theta}^{0}, \boldsymbol{G}^{0},\boldsymbol{P}^{0})   }{\lambda3}$.
		Meanwhile, it could be derived from Theorem 2 that $\boldsymbol{G}^*$ is bounded. By choosing $\psi_1 =\sqrt{\frac{ \mathcal{F}(\boldsymbol{\theta}^{0},  \boldsymbol{G}^{0},\boldsymbol{P}^{0})   }  {\lambda1}}, \psi_2 =\frac{\mathcal{F}(\boldsymbol{\theta}^{0}, \boldsymbol{G}^{0} ,\boldsymbol{P}^{0})  }{\lambda_2}, \psi_3= \frac{\mathcal{F}(\boldsymbol{\theta}^{0}, \boldsymbol{G}^{0},\boldsymbol{P}^{0})   }{\lambda3} $, all solutions chosen by our algorithm  belongs to $\Theta$. Then, with high probability, all these solutions could reach a reasonable generation gap between the expected 0-1 AUC loss metric $\mathbb{E}_\mathcal{D}(\sum_i\ell^{(i)}_{AUC})$ and the estimated surrogate loss on the training data $\mathcal{L}(W)$, with an order $\mathcal{O}(\sum_{i=1}^U\frac{1}{\sqrt{(n_i\chi_i(1-\chi_i))}})$.
	\end{rem}
	
	\section{Empirical Study}
	\subsection{Experiment Settings}
	For all the experiments, hyper-parameters are tuned based on the training and validation set(account for 85\% of the total instances), and the result on the test set are recorded. The experiments are done with 15 repetitions for each involved algorithm.     
	\subsection{Competitors}
	In this paper, we compare our model with the following competitors:
	\textbf{Robust Multi-Task Learning (RMTL)} \cite{RMTL}: RMTL aims at identifying irrelevant tasks when learning from multiple tasks. To this end, the model parameter is decomposed into a low-rank structure and group sparse structure. 
	\textbf{Robust Multi-Task Feature Learning (rMTFL)} \cite{rMTFL}: rMTFL assumes that the model W can be decomposed into two components: a shared feature structure $\boldsymbol{P}$($\ell_{1,2}$ norm penalty)  and a group-sparse structure $\boldsymbol{Q}$ ($\ell_{1,2}$ norm penalty on its transpose) that detects outliers. 
	\textbf{Lasso}: The the $\ell_1$-norm regularized multi-task least squares method. 
	\textbf{Joint Feature Learning (JFL)}\cite{JFL}: In JFL all the models are expected to share a common set of features. To this end, the group sparsity constraint is imposed on the models via the $\ell_{1,2}$ norm.
	\textbf{The Clustered Multi-Task Learning Method (CMTL)}: \cite{CMTL}: CMTL assumes that the tasks could be clustered into $k$ groups. Then a k-means based regularizer is adopted to leverage such a structure.
	\textbf{The task-feature co-clusters based multi-task method (COMT)} \cite{cocluster}:  COMT assumes that the task-specific components bear a feature-task coclustering structure.
	\textbf{Reduced Rank Multi-Stage multi-task learning (RAMU)} \cite{RAMUSA}: RAMU adopts a capped trace norm regularizer to minimize only the singular values smaller than an adaptively tuned threshold.  
	
	Note that since \cite{user1} adopts an extra data pool and \cite{user2} includes extra initialization algorithms based on \cite{user1}, our method is not compared with them for the sake of fairness.

	\begin{table}[htbp]
		\centering
		\caption{AUC Comparison on Simulation Dataset}
		\begin{tabular}{l|cccc}
			\toprule
			Alg & RMTL & rMTFL & LASSO & JFL \\
			\midrule
			mean & 83.48  & 83.45  & 83.57  & 83.49  \\
			\midrule
			Alg & CMTL & COMT & RAMU & Ours \\
			\midrule
			mean & 83.47  & 83.44  & 83.50  & {\textbf{99.65 }} \\
			\bottomrule
		\end{tabular}%
		\label{tab:addlabel}%
	\end{table}%
	
	\begin{table}[htbp]
		\centering
		\caption{Running Time Comparison (seconds): Original stands for the original AUC evaluation, wheres ours stands for our acceleration scheme.}
		\begin{tabular}{l|ccccc}
			\toprule
			ratio & 20\% & 40\% & 60\% & 80\% & 100\% \\
			\midrule
			Orginal & 18.57 & 74.22  & 151.86 & 268.55  & nan  \\
			Ours & {\textbf{3.06}} & {\textbf{5.50}} & {\textbf{8.65}} & {\textbf{12.46}} & {\textbf{15.82}} \\
			\bottomrule
		\end{tabular}%
		\label{tab:time}%
	\end{table}%
	
	\subsection{Simulated Dataset}
	\begin{table*}[htbp]
		\centering
		
		\caption{Performance Comparison based on the AUC metric }
		\begin{tabular}{l|ccccccccccccc}
			\toprule
			\multicolumn{1}{c|}{\multirow{3}[6]{*}{Alg}} & \multicolumn{13}{c}{Attibutes} \\
			\cmidrule{2-14}    \multicolumn{1}{c|}{} & \multicolumn{7}{c}{Shoes}               &     & \multicolumn{5}{c}{Sun} \\
			\cmidrule{2-8}\cmidrule{10-14}    \multicolumn{1}{c|}{} & BR  & CM  & FA  & FM  & OP  & ON  & PT  &     & CL  & MO  & OP  & RU  & SO \\
			\cmidrule{1-8}\cmidrule{10-14}    RMTL & 79.31  & 84.99  & 66.90  & 85.08  & 75.67  & 67.22  & 75.14  &     & 69.36  & 62.71  & 75.28  & 67.91  & 69.23  \\
			rMTFL & 70.90  & 83.78  & 67.27  & 85.91  & 73.71  & 65.21  & 77.11  &     & 69.27  & 62.15  & 75.80  & 68.16  & 68.76  \\
			LASSO & 68.46  & 80.48  & 65.90  & 84.01  & 71.47  & 64.60  & 75.08  &     & 67.64  & 61.83  & 75.39  & 68.57  & 69.13  \\
			JFL & 72.00  & 83.10  & 67.26  & 85.93  & 73.02  & 65.39  & 77.09  &     & 68.63  & 61.94  & 75.00  & 67.17  & 68.78  \\
			CMTL & 74.54  & 85.16  & 68.21  & 85.32  & 75.06  & 68.17  & 77.62  &     & 72.55  & 66.61  & 79.78  & 72.34  & 72.82  \\
			COMT & 84.24  & 88.68  & 69.66  & 89.19  & 80.93  & 72.99  & 80.62  &     & 70.69  & 63.72  & 76.93  & 69.43  & 70.44  \\
			RAMU & 78.33  & 84.58  & 65.78  & 84.68  & 75.25  & 66.72  & 73.50  &     & 72.95  & 69.25  & 79.81  & 74.39  & 72.50  \\
			Ours & {\textbf{92.95}} & {\textbf{90.92}} & {\textbf{73.24 }} & {\textbf{92.65}} & {\textbf{87.95}} & {\textbf{81.07}} & {\textbf{86.22}} &     & {\textbf{79.31}} & {\textbf{78.19}} & {\textbf{86.50}} & {\textbf{81.88}} & {\textbf{78.98}} \\
			\bottomrule
		\end{tabular}%
		\label{tab:perfall}%
	\end{table*}%

	\begin{figure}[ht]
		\centering
		\subfloat[Ground-Truth]{
			\includegraphics[width =0.45\columnwidth]{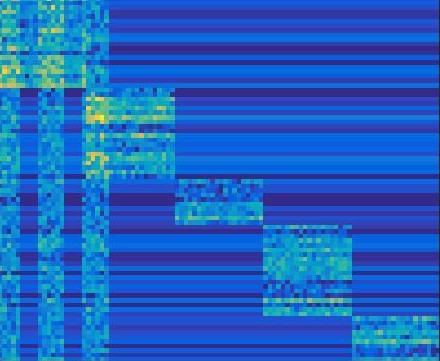}
		}
		\subfloat[ Learned Parameter ]{
			\includegraphics[width=0.45\columnwidth]{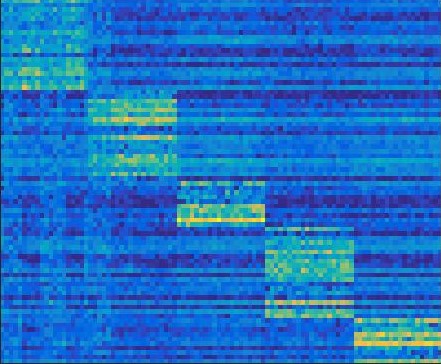}
		}
		\caption{The Potential of our proposed method to Recover the Expected Structure of the Parameters }
		\label{fig:struc}
	\end{figure}
	
	\begin{figure}[ht]
		\centering
		\subfloat[Loss Convergence]{
			\includegraphics[width =0.45\columnwidth]{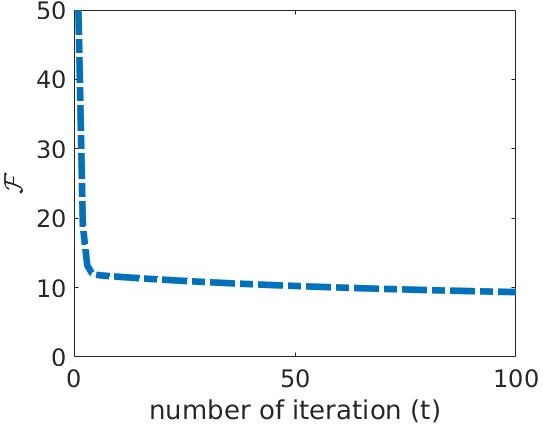}
		}
		\subfloat[Parameter Convergence ]{
			\includegraphics[width=0.45\columnwidth]{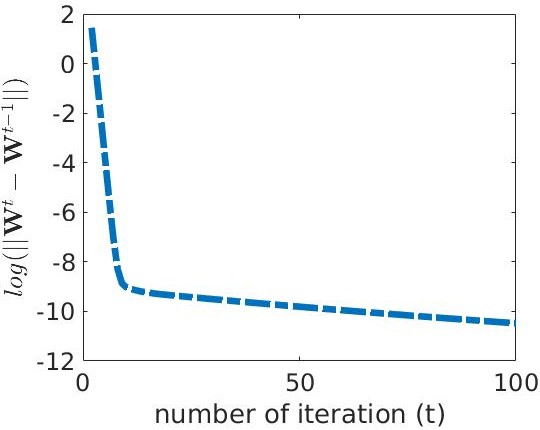}
		}
		\caption{The Convergence Behavior On Simulation Dataset: a)shows the loss convergence, whereas b) exhibits the convergence property in terms of the parameters. }
		\label{fig:iter}
	\end{figure}

	\begin{figure}[ht]
		\centering
		\subfloat[Shoes Dataset]{\includegraphics[width= 0.48\columnwidth]{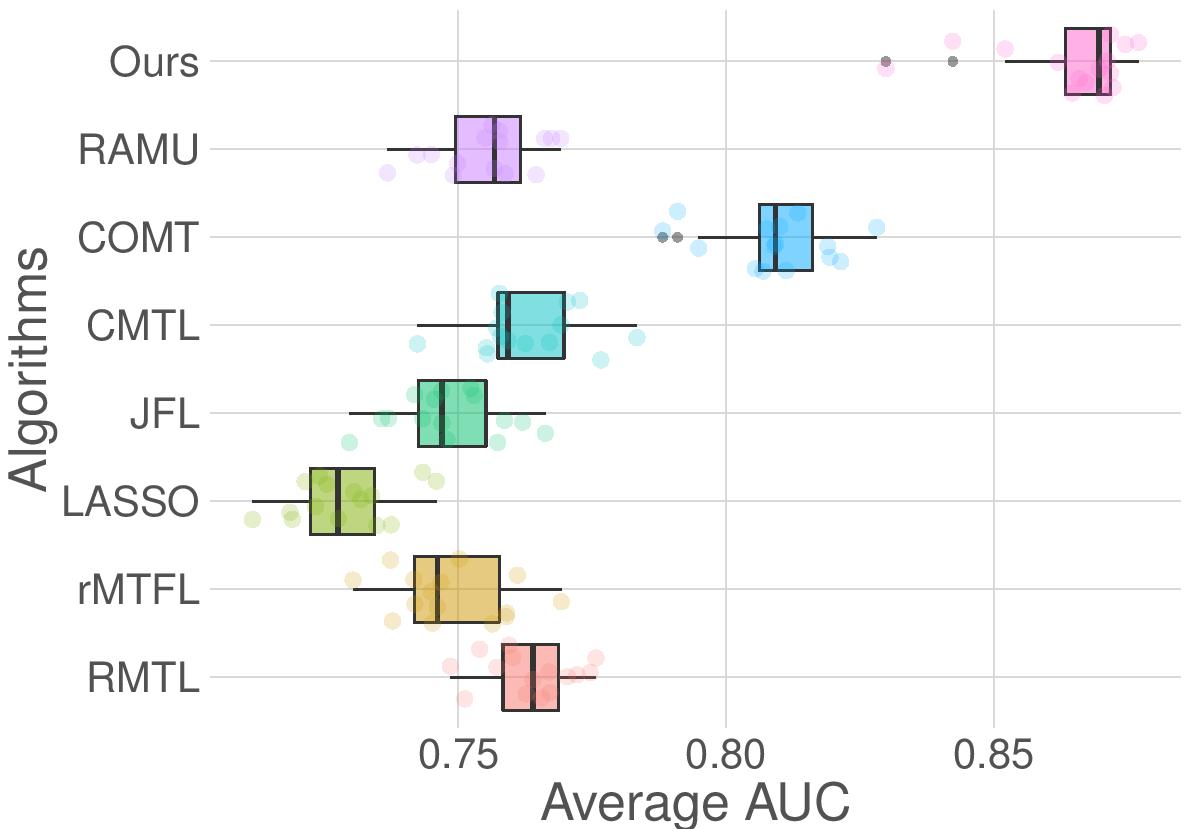}}
		\subfloat[Sun Attribute Dataset]{\includegraphics[width= 0.48\columnwidth]{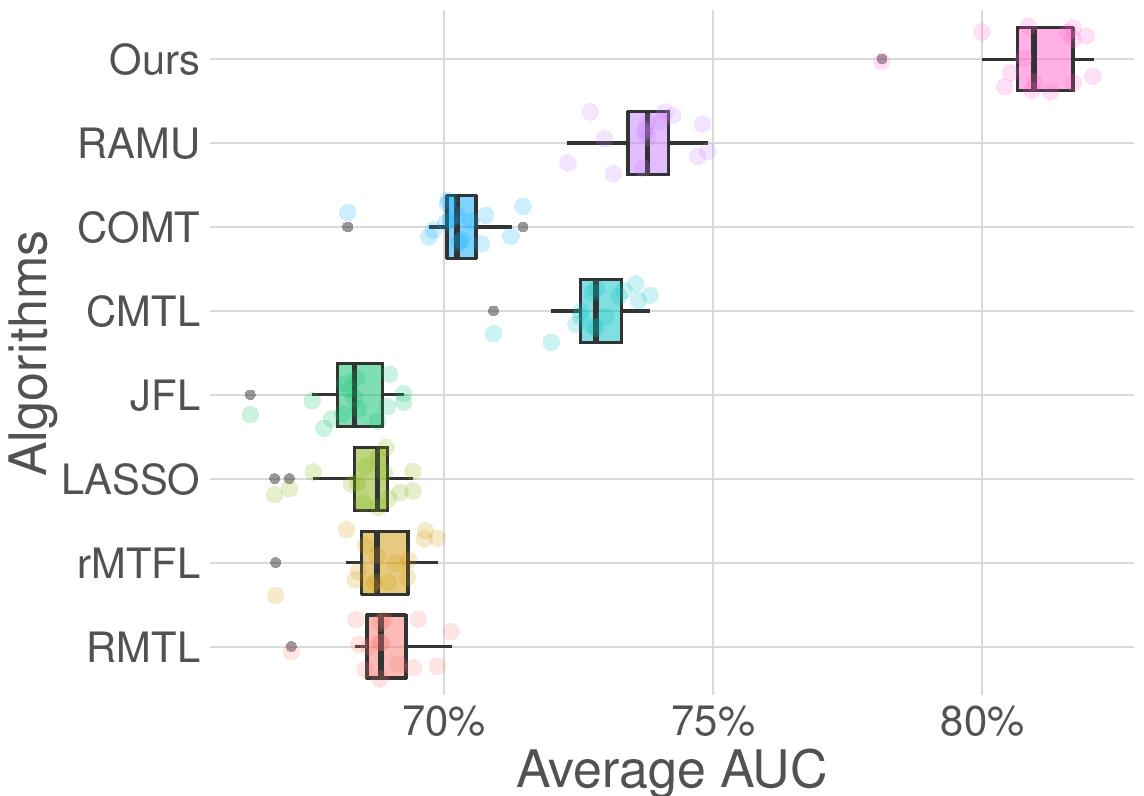}}		
		\caption{Average performances on all attributes of shoes dataset}
		\label{fig:box}
	\end{figure}
    In this subsection, we will generate a simulated annotation dataset with 100 simulated users, where the features and AUC scores are produced according to our proposed model. For each user, 5000 samples are generated as $\boldsymbol{X}^{(i)} \in \mathbb{R}^{5000\times80}$ and $\boldsymbol{x}^{(i)}_k \sim \mathcal{N}(0,\boldsymbol{I}_{80})$. This results in a volume of \textit{500,000 overall annotations}. To capture the global information, we set $\boldsymbol{\theta}$ as $\boldsymbol{\theta} \sim \mathcal{U}(0,5) + \mathcal{N}(0,0.5^2)$. In terms of the co-cluster nature, $\boldsymbol{G}$ is produced with a block-wise grouping structure for feature-user co-cluster.  Specifically, we create 5 blocks for $\boldsymbol{G}$, namely: $\boldsymbol{G}(1:20,1:20)$, $\boldsymbol{G}(21:40,21:40)$, $\boldsymbol{G}(41:50,41:60)$, $\boldsymbol{G}(51:70,61:80)$ and $\boldsymbol{G}(71:80,81:100)$. For each of the block, the elements are generated from the distribution $\mathcal{N}(C_i, 2.5^2) $ (generated via element-wise sampling) where $C_i \sim \mathcal{U}(0,10)$ is the centroid for the corresponding cluster and thus is shared among a specific cluster. For the elements that do not belong to the 5 chosen blocks are set as 0. For $\boldsymbol{P} \in \mathbb{R}^{d \times U}$, we set $\boldsymbol{P}(:,1:5)$, $\boldsymbol{P}(:,10:15)$, $\boldsymbol{P}(:,20:25)$ randomly with the distribution $\mathcal{U}(0,10)$, while the remaining entries are set as 0. For each user, the scoring function are generated as $\boldsymbol{s}^{(i)}= \boldsymbol{X}^{(i)}(\boldsymbol{\theta}+ \boldsymbol{G}^{(i)} + \boldsymbol{P}^{(i)}) +\epsilon^{(i)}$, where $\epsilon^{(i)} \in \mathbb{R}^{5000 
		\times 1}$, and $\epsilon^{(i)} \sim \mathcal{N}(0,0.01^2\boldsymbol{I}_{5000})$. To generate the labels $\boldsymbol{Y}^{(i)}$ for each $i$, the top 100 instances with highest scores are labeled as 1, while the remaining instances are labeled as -1. \\
	\indent The performance of all the involved algorithms on the simulated dataset is recorded in Table 1. The corresponding results show that our proposed algorithm consistently outperforms other competitors. Specifically, our algorithm reaches an AUC score of 99.65, where the second best algorithm only attains a score of 83.57.\\
	\indent Besides the generalized performance,  we could also verify empirically the ability of our algorithm to recover the expected structures on the parameters $\boldsymbol{W}$. With the same simulated dataset, we compare the  parameter $\boldsymbol{W}$ learned from our proposed and the Ground Truth parameters in Figure \ref{fig:struc}. The results show that our proposed methods could roughly recover the expected group-based structure.\\
	\indent In Theorem 2, we have proved the convergence behavior of the proposed algorithm. To verify theoretical findings, we plot the loss and parameter evolution against the number of iteration in Figure \ref{fig:iter}. In Figure \ref{fig:iter}-(a), we see that the loss function constantly decreases as the iteration proceeds, whereas in Figure \ref{fig:iter}-(b), it is easy to find that the parameter difference $log(\norm{\boldsymbol{W}^{t+1} - \boldsymbol{W}^{t}})$ also keeps decreasing. All these empirical observations coincide with our theoretical results.\\
	\indent To verify the efficiency of the proposed AUC evaluation scheme, we evaluate the running time of our algorithm with and without the AUC evaluation scheme. The resulting comparison is recorded in Table \ref{tab:time} when different ratios of the dataset are adopted as the training set. As what exhibited here, the original algorithm without the efficient AUC evaluation scheme gets slowly sharply when the training sample increases. When 100\% samples are included in the training set, our server couldn't finish the program within 1h due to the memory limit (24GB). We denote as nan correspondingly. In contrast, we can see an up to 20 times speed-up with the help
	of the proposed scheme in proposition 1. 
	\subsection{Shoes Dataset}
	The Shoes Dataset is collected from \cite{user2} which contains 14,658 online shopping images. In this dataset, 7 attributes are annotated by users with a wide spectrum of interests and backgrounds. For each attribute, there are at least 190 users who take part in the annotation, and each user is assigned with 50 images. Overall, 90,000 annotations are collected in this dataset. We concatenate the GIST and color histograms provided by the original dataset as the features. To remove the redundant input features,  Principal  Component Analysis (PCA) is performed before training, and only the components that are capable of interpreting the first 99\% of the total data variance are preserved. Meanwhile, we also need to eliminate the effect of the users who extremely prefer to provide merely the positive (negative) labels. To this end, we remove the users who give less than 8 annotations for at least one of the classes.\\
	\indent The left half of Table \ref{tab:perfall} shows the average performance of the 15 repetitions with the experimental setting (BR: Brown, CM: Comfortable, FA: Fashionable, FM: Formal, OP: Open, ON: Ornate, PT: Pointy). Furthermore, in Figure \ref{fig:box} ,we visualize the average result over the 7 attributes for 15 repetitions with a boxplot. Accordingly, we could reach the  conclusion that our proposed algorithm consistently outperforms all the benchmark algorithms by a significant margin.
	\subsection{Sun Attribute Dataset}
	The SUN Attributes Dataset \cite{sundata}, is a well-known large-scale scene attribute dataset with roughly 1,4000 images and  a taxonomy of 102 discriminative attributes. Recently, in \cite{user2}, the personalized annotations over five attributes are collected with hundreds of annotators. For each person, 50 images are labeled based on their own comprehension and preference. Overall, this dataset contains 64,900 annotations collected from different users. As for dataset preprocessing, we adopt almost the same procedure as the shoes dataset. The difference here is that we use the second last fc layer of the Inception-V3 \cite{inception} network as the input feature. Furthermore, the PCA~is done for each attribute preserving 90\% of the total data variance.\\
	\indent  The right half of Table \ref{tab:perfall} shows the average performance over 15 repetitions (CL: Cluttered, MO: Modern, OP: Opening Area, RU: Rustic, SO: Soothe) , and Figure \ref{fig:box}-(b) shows the average AUC scores over 5 attributes for the 15 repetitions. Similar to the shoes dataset, we see that our proposed algorithm consistently outperforms all the benchmark algorithms.
	\section{Conclusion}
	In this paper, we propose a novel multi-task model for learning user-specific attribute comprehension with a hierarchical decomposition to model the consensus-to-personalization evolution and an AUC-based loss function to learn the preference. Furthermore, we propose an efficient AUC-based evaluation method to significantly reduce the computational complexity of computing the loss and the gradients. Eventually, both the theoretical results and the experimental results demonstrate the effectiveness of our proposed model.
	\section{Acknowledgment}
	This work was supported by the National Key R\&D Program of China (Grant No. 2016YFB0800603). The research of Zhiyong Yang and Qingming Huang was supported in part by National Natural Science Foundation of China:  61332016, 61650202 and 61620106009, in part by Key Research Program of Frontier Sciences, CAS: QYZDJ-SSW-SYS013. The research of Qianqian Xu was supported in part by National Natural
	Science Foundation of China (No.61672514, 61390514, 61572042),
	Beijing Natural Science Foundation (4182079), Youth Innovation
	Promotion Association CAS, and CCF-Tencent Open Research Fund.
	The research of Xiaochun Cao was supported by National Natural Science Foundation of China (No. U1636214, 61733007, U1605252), Key Program of the Chinese Academy of Sciences (No. QYZDB-SSW-JSC003).


	\bibliographystyle{named}
	\bibliography{aaai19_zhiyong_full}
	\newpage
	
\onecolumn
\section{\LARGE Supplementary Materials } 

 \section{Efficient AUC Evaluations}
 \subsection{Proof of Proposition 1}
 According to the definition of $\mathcal{G}_i$ and $\varepsilon_i$, the affinity matrix of $\mathcal{G}_i$ could be written as 
 \begin{equation*}
 \mathcal{W}_i = \frac{1}{n_+n_-}[\boldsymbol{\tilde{y}^{(i)}}(\boldsymbol{1}-\boldsymbol{\tilde{y}^{(i)}})^\top + (1-\boldsymbol{\tilde{y}^{(i)}})(\boldsymbol{\tilde{y}^{(i)}})^\top ]. 
 \end{equation*}
 Correspondingly, $\boldsymbol{L}^{(i)}$ could be simplified as :
 \begin{equation}\label{key}
 \boldsymbol{L}^{(i)} = diag\left(\frac{\boldsymbol{\tilde{y}^{(i)}}}{n_{+,i}}+ \frac{\boldsymbol{1} - \boldsymbol{\tilde{y}^{(i)}}}{n_{-,i}}\right) -\mathcal{W}_i
\end{equation}
Then for $\boldsymbol{A}^\top\boldsymbol{L}^{(i)}\boldsymbol{B}$, we have:
\begin{equation}\label{key}
\boldsymbol{A}^\top\boldsymbol{L}^{(i)}\boldsymbol{B} = \boldsymbol{A}^\top\Bigg(diag\left(\frac{\boldsymbol{\tilde{y}^{(i)}}}{n_{+,i}}+ \frac{\boldsymbol{1} -
	 \boldsymbol{\tilde{y}^{(i)}}}{n_{-,i}}\right)\Bigg)\boldsymbol{B} - \boldsymbol{\tilde{A}}_+\boldsymbol{\tilde{B}}_-^\top - \boldsymbol{\tilde{A}}_-\boldsymbol{\tilde{B}}_+^\top
\end{equation}
where 
\begin{equation}
\boldsymbol{\tilde{A}}_+ =\frac{1}{n_{+,i}} \boldsymbol{A}^\top\boldsymbol{\tilde{y}}^{(i)},\  \ \boldsymbol{\tilde{A}}_- =\frac{1}{n_{-,i}} \boldsymbol{A}^\top(\boldsymbol{1} -\boldsymbol{\tilde{y}}^{(i)}),  \ \boldsymbol{\tilde{B}}_+ =\frac{1}{n_{+,i}} \boldsymbol{B}^\top\boldsymbol{\tilde{y}}^{(i)}, \ \boldsymbol{\tilde{B}}_- =\frac{1}{n_{-,i}} \boldsymbol{B}^\top(\boldsymbol{1} -\boldsymbol{\tilde{y}}^{(i)})
\end{equation}
It is easy to see that if we first cache $\boldsymbol{\tilde{A}}_+, \ \boldsymbol{\tilde{A}}_-$ and $\boldsymbol{\tilde{B}}_+, \boldsymbol{\tilde{B}}_-$, then $\boldsymbol{A}^\top\boldsymbol{L}^{(i)}\boldsymbol{B}$ could be computed within $\mathcal{O}(abn_i + (a+b)n_i)$, which should cost at most $\mathcal{O}(abn_i^2)$ with a naive method.
Similarly, for $\boldsymbol{A}^\top\boldsymbol{L}$, we have the following simplification:
\begin{equation}\label{key}
\boldsymbol{A}^\top\boldsymbol{L} = \boldsymbol{A}^\top\Bigg(diag\left(\frac{\boldsymbol{\tilde{y}^{(i)}}}{n_{+,i}}+ \frac{\boldsymbol{1} - \boldsymbol{\tilde{y}^{(i)}}}{n_{-,i}}\right)\Bigg) - \boldsymbol{\tilde{A}}_+\Bigg(\frac{\boldsymbol{\boldsymbol{1}-\tilde{y}^{(i)}}}{n_{-,i}}^\top\Bigg) - \boldsymbol{\tilde{A}}_-\Bigg(\frac{\boldsymbol{\tilde{y}^{(i)}}}{n_{+,i}}^\top\Bigg)
\end{equation}
This shows that the complexity for computing $\boldsymbol{A}^\top\boldsymbol{L}^{(i)}$ could be reduced from at most $\mathcal{O}(n_i^2a)$ to $\mathcal{O}(n_ia)$
 \qed
 \subsection{Efficient method to evaluation the AUC loss and gradients}
\subsubsection{Evaluation of The Loss}  For a specific user $i$, we denote $\boldsymbol{\tau}_i = (\boldsymbol{\tilde{y}}^{(i)} - \boldsymbol{X}^{(i)}\boldsymbol{W}^{(i)} ), \in \mathbb{R}^{n_i}$, then the empirical loss for AUC could be represented as : $\ell_i  = \boldsymbol{\tau}_i^\top\boldsymbol{L}^{(i)} \boldsymbol{\tau}_i$.
Letting $\boldsymbol{A} = \boldsymbol{B} = \boldsymbol{\tau}_i$, we could reduce the time complexity of calculating $\ell_i$ from $\mathcal{O}(n_{+,i}n_{-,i})$ to $\mathcal{O}(n_{+,i}+n_{-,i})$     using the details in the proof of Proposition 1.
\subsubsection{Evaluation of The Gradients}
Denote $\Delta_{i}$ as :
\begin{equation}
\begin{split}
\Delta_{i} = \boldsymbol{X}^{(i)\top}\boldsymbol{L}^{(i)}\boldsymbol{X}^{(i)}(\boldsymbol{\theta}+\boldsymbol{G}^{(i)}  + \boldsymbol{P}^{(i)} )- \boldsymbol{X}^{(i)\top}\boldsymbol{L}^{(i)}\boldsymbol{y}^{(i)}
\end{split}
\end{equation}
We could reach :
\begin{align*}
&\nabla_{\boldsymbol{\theta}}\mathcal{L}(\boldsymbol{W}) =
\sum\limits_{i=1}^{U}\Delta_{i}, \ \ 
\nabla_{\boldsymbol{P}^{(i)}}\mathcal{L}(\boldsymbol{W}) =\nabla_{\boldsymbol{G}^{(i)}}\mathcal{L}(\boldsymbol{W})
=\Delta_{i}.
\end{align*}
Obviously, $\boldsymbol{X}^{(i)\top}\boldsymbol{L}^{(i)}\boldsymbol{y}^{(i)}$ could be finished within $\mathcal{O}(dn_i)$ with Proposition 1. For $\boldsymbol{X}^{(i)\top}\boldsymbol{L}^{(i)}\boldsymbol{X}^{(i)}(\boldsymbol{\theta}+\boldsymbol{G}^{(i)}  + \boldsymbol{P}^{(i)} )$, if  computation is finished in the following order:
\begin{equation}
\big(\boldsymbol{X}^{(i)\top}\boldsymbol{L}^{(i)}\big)\cdot\big(\boldsymbol{X}^{(i)}(\boldsymbol{\theta}+\boldsymbol{G}^{(i)}  + \boldsymbol{P}^{(i)} )\big).
\end{equation} 
We could still finish the computation of this term with a complexity $\mathcal{O}(dn_i)$. Above all, in light of Proposition 1, the complexity for  computing the gradients could be efficiently reduced from $\mathcal{O}(d\sum_i n_{+,i} \cdot n_{-,i})$ to $\mathcal{O}(d\sum_i (n_{+,i} + n_{-,i}))$.
\subsection{Proof of Theorem 1}

We further denote 
\begin{equation*}
dL_{i} = \boldsymbol{X}^{(i)\top}\boldsymbol{L}^{(i)}\boldsymbol{X}^{(i)}(\boldsymbol{W}^{(i)} - \boldsymbol{W'}^{(i)}),\
d\boldsymbol{\theta} = \nabla_{\boldsymbol{\theta}}\mathcal{L}(\boldsymbol{W}) -\nabla_{\boldsymbol{\theta'}}\mathcal{L}(\boldsymbol{W'}),\ d\boldsymbol{P}^{(i)} = \nabla_{\boldsymbol{P}^{(i)}}\mathcal{L}(\boldsymbol{W}) -\nabla_{\boldsymbol{P'}^{(i)}}\mathcal{L}(\boldsymbol{W'}).
\end{equation*}
Note that:
\begin{equation*}
\boldsymbol{W}^{(i)} = \boldsymbol{\theta}+\boldsymbol{G}^{(i)}  + \boldsymbol{P}^{(i)}, \
\boldsymbol{W}'^{(i)} = \boldsymbol{\theta'}+\boldsymbol{G'}^{(i)}  + \boldsymbol{P'}^{(i)}
\end{equation*}
It thus follows that 
\begin{equation}\centering
\begin{split}
&~~~\norm{\nabla{\mathcal{L}(\boldsymbol{W})} - \nabla{\mathcal{L}(\boldsymbol{W}')}}\\
=&~~\left(\norm{d\boldsymbol{\theta}}^2+ \sum\limits_{i}\norm{d\boldsymbol{G}^{(i)}}^2 + \sum\limits_{i}\norm{d\boldsymbol{P}^{(i)}}^2 \right)^{1/2}\\ {\leq}&~~\norm{d\boldsymbol{\theta}}+\sum\limits_{i}\norm{d\boldsymbol{G}^{(i)}} + \sum\limits_{i}\norm{d\boldsymbol{P}^{(i)}} \\
 =&~~ \norm{\sum\limits_{i=1}^{U}dL_{i}} + 2\sum\limits_{i=1}^{U}
\norm{dL_{i}}
\\\leq&~ ~3C\sum\limits_{i=1}^{U}\norm{\boldsymbol{W}^{(i)}- \boldsymbol{W}'^{(i)}}\\ \leq&~~ 3C  \Big(U\norm{\boldsymbol{\theta} -\boldsymbol{\theta'}} +\sum_{i=1}^{U}\norm{\boldsymbol{G}^{(i)} -\boldsymbol{G'}^{(i)}}+\sum_{i=1}^{U}\norm{\boldsymbol{P}^{(i)} -\boldsymbol{P'}^{(i)}}  \Big)\\ {\leq}& ~~
3CU\sqrt{2U+1}\norm{vec(\boldsymbol{W}) - vec(\boldsymbol{W}') }\\ \leq & ~ ~ \gamma \norm{vec(\boldsymbol{W}) - vec(\boldsymbol{W}') }
\end{split}
\end{equation}
where
\begin{equation*}
C = \max\limits_{i}\left(\norm{\boldsymbol{X}^{(i)^\top}\boldsymbol{L}^{(i)}\boldsymbol{X}^{(i)}}_2\right).
\end{equation*} 
To prove the last  inequality, we have :
\begin{equation}\label{finieq1}
\forall i, \norm{\boldsymbol{X}^{(i)^\top}\boldsymbol{L}^{(i)}\boldsymbol{X}^{(i)}}_2 \le \norm{\boldsymbol{X}^{(i)}}^2_2\norm{\boldsymbol{L}^{(i)}}_2.
\end{equation} 
and according to Theorem 3.3 of \cite{graph}, we have   
\begin{equation}\label{finieq2}
\norm{\boldsymbol{L}^{(i)}}_2=\frac{n_i}{n_{+,i}n_{-,i}}
\end{equation} 
Combining (\ref{finieq1}) and (\ref{finieq2}), the last inequality holds, thus we complete the proof.
 \section{Convergence Analysis}
  \subsection{Proof of Proposition 2}
For the sake of simplicity, we denote $\mathcal{T}_{\kappa,\frac{\lambda_3}{\rho_k}}(\boldsymbol{\Sigma})$ here in this proof as $\mathcal{T}$, $T_{ii}$ as the $i$-th diagonal entry of $\mathcal{T}$ i.e. $(\frac{1}{2c+1})^{I[i>\kappa]}\sigma_i$, and $m$ as $min\{d,U\}$. \\
According to the result of  General Singular Value Thresholding(GSVT)  \cite{gsvt}, $G^*$ is an optimal solution of (\ref{Pg}), if $\mathcal{T}_{11}, \cdots \mathcal{T}_{mm}  $  is a minimizer of the following problem: 
 \begin{equation}\label{pro:cons}
  \argmin\limits_{\iota_1 \ge \iota_{2}, \cdots, \ge \iota_{m}\ge 0} \frac{1}{2}\sum\limits_{i}(\iota_i-\sigma_i(\boldsymbol{\tilde{G}}^k) )^2 +\frac{\lambda_3}{\rho^2_k}\sum\limits_{i=\kappa+1}^{min\{d,U\}}\iota^2_i
 \end{equation}
 It is easy to see that $\mathcal{T}_{11}, \cdots \mathcal{T}_{mm}  $ is the unique minimizer of the unconstrained problem
  \begin{equation}\label{key}
 \argmin \frac{1}{2}\sum\limits_{i}(\iota_i-\sigma_i(\boldsymbol{\tilde{G}}^k) )^2 +\frac{\lambda_3}{\rho^2_k}\sum\limits_{i=\kappa+1}^{min\{d,U\}}\iota^2_i,
 \end{equation}
 due to its strong convexity.\\
  It only remains to check that $\mathcal{T}_{11}, \cdots \mathcal{T}_{mm}  $ is also a feasible solution to the constrained problem. Given the group number $\kappa$, for any $ 1 \le i\le min\{d,U\}, 1 \le j \le min\{d,U\} $, and $i\neq j$, we have :
  \begin{equation*}
  (\mathcal{T}_{ii} - \mathcal{T}_{jj})(\sigma_i - \sigma_j) = \Bigg(\left({2\frac{\lambda_3}{\rho_k}+1}\right)^{-I[i>\kappa]}\sigma_i -  \left({2\frac{\lambda_3}{\rho_k}+1}\right)^{-I[j>\kappa]}\sigma_j \Bigg)\Bigg(\sigma_i - \sigma_j\Bigg)
  \end{equation*}
  We have
 \begin{itemize}
\item  If $i < \kappa$ and $j < \kappa$ : ~$(\mathcal{T}_{ii} - \mathcal{T}_{jj})(\sigma_i - \sigma_j) =(\sigma_i - \sigma_j)^2 \ge 0 $
\item If one of $i, j$ is greater than $\kappa$  :~$(\mathcal{T}_{ii} - \mathcal{T}_{jj})(\sigma_i - \sigma_j) > (\sigma_i - \sigma_j)^2 \ge 0 $.
\item If both $i$ and $j$ are greater than $\kappa$ : ~$(\mathcal{T}_{ii} - \mathcal{T}_{jj})(\sigma_i - \sigma_j) = (2\frac{\lambda_2}{\rho_k}+1)^{-1} (\sigma_i - \sigma_j)^2 \ge 0 $
 \end{itemize}
Thus $(\mathcal{T}_{ii} - \mathcal{T}_{jj})(\sigma_i - \sigma_j) \ge 0$. This implies that the diagonal elements of $\mathcal{T}$ preserve the order of the singular values of $\boldsymbol{\tilde{G}}^k$. This proves that $\mathcal{T}$ is also a feasible solution of (\ref{pro:cons}). Finally, we have $\mathcal{T}$ is the minimizer of (\ref{pro:cons}) and complete the proof.
 \qed
\subsection{Line Search Strategy}
As we mentioned in the main paper, the line search strategy is a crucial component  which guarantees the effectiveness of (\ref{upper bound}) via tuning the parameter $\rho_k$. Now we provide a more detailed review of this strategy mentioned in the main paper.
\begin{algorithm}[ht]           
	\caption{ $LineSearch\left(\rho_0, \alpha,  \boldsymbol{\tilde{\theta}}^k, \boldsymbol{\tilde{G}}^k, \boldsymbol{\tilde{P}}^k\right)$ }
	\KwIn{$\rho_0$, $\alpha>1$, $\boldsymbol{\tilde{\theta}}^k, \boldsymbol{\tilde{G}}^k, \boldsymbol{\tilde{P}}^k$}
	
	\KwOut{$\rho$}

	Initialize $\rho = \rho_0$\;
	\While{True}{
		Solve  $(\boldsymbol{\theta}^k,\boldsymbol{G}^k,\boldsymbol{P}^k)$ with 
		$\rho_k= \rho $ and ($\boldsymbol{\tilde{\theta}}^k, \boldsymbol{\tilde{G}}^k, \boldsymbol{\tilde{P}}^k$)\;
		\uIf{$\lk{k} <
			\mathcal{L}(\boldsymbol{\tilde{\theta}}^k, \boldsymbol{\tilde{G}}^k, \boldsymbol{\tilde{P}}^k) + \Psi_{\rho_k}(D\boldsymbol{\theta})
			+  \Psi_{\rho_k}(D\boldsymbol{G}) +  \Psi_{\rho_k}(D\boldsymbol{P})$}{break\;
			}\Else{$\rho= \rho \cdot \alpha$} 
		
	}
\end{algorithm} 
 \subsection{A Summary of the Optimization Method}
With the aforementioned closed-form solution for $\boldsymbol{G}$ and line search strategy to find a proper $ \rho_k $. For the sake of clarity,  we now summarize the details of the proximal gradient method in Algorithm 2.
\begin{algorithm}[ht]           
	\caption{The  proximal gradient method }
	\KwIn{$\mathcal{D}$, $\lambda_1$,$\lambda_2$,$\lambda_3$,$\rho_0>0$ ,$\alpha >1$}
	
	\KwOut{$\boldsymbol{\theta}^{k-1}$,$\boldsymbol{G}^{k-1}$ ,$\boldsymbol{P}^{k-1}$}

	Initialize $\boldsymbol{\theta}^0$,$\boldsymbol{G}^0$, $\boldsymbol{P}^0$, \;
	$\boldsymbol{\theta}^{ref} :=\boldsymbol{\theta}^{0},\boldsymbol{G}^{ref} :=\boldsymbol{G}^{0},\boldsymbol{P}^{ref} :=\boldsymbol{P}^{0} $\;
	\While{Not Converged}{
		Solve  $(\boldsymbol{\theta}^k,\boldsymbol{G}^k,\boldsymbol{P}^k)$ from Eq.(\ref{Ptheta})-Eq.(\ref{Pp})\;
		Set $\rho_k:= LineSearch(\rho_{k-1},\alpha,\boldsymbol{\tilde{\theta}}^k, \boldsymbol{\tilde{G}}^k, \boldsymbol{\tilde{P}}^k) $ \;
		Update$(\boldsymbol{\theta}^k,\boldsymbol{G}^k,\boldsymbol{P}^k)$ again with $\rho_k$ \; 
		$\boldsymbol{\theta}^{ref}:= \boldsymbol{\theta}^{k} $; \ 	$\boldsymbol{G}^{ref}:= \boldsymbol{G}^{k} $; \
		$\boldsymbol{P}^{ref}:=
		\boldsymbol{P}^k$\;
		$k = k+1$\;
	}
\end{algorithm} 
 \subsection{Definition of the sub-differential}
 Now we introduce the generalized sub-differential for proper and lower semi-continues functions, and the readers are referred to \cite{var} for a more detailed analysis. In this section, we denote $\mathcal{L}(\boldsymbol{W})$ as $\mathcal{L}(\boldsymbol{\theta}, \boldsymbol{G}, \boldsymbol{P})$ for clarity.
 \begin{defi}
 Consider a function $f: \mathbb{R}^n \rightarrow \bar{\mathbb{R}}$ and a point $\bar{x}$ with $f(\bar{x})$ finite. For a vector $v \in \mathbb{R}^n$, we have:
 \begin{itemize}
 	\item a) $v$ is a \textit{\textbf{regular subgradient}} of $f$ at $\bar{x}$, written $v \in {\hat{\partial} f(\bar{x})}$, if 
 	\begin{equation}\label{key}
 	f(x) \ge f(\bar{x}) + \left<v, x-\bar{x}\right> + o(|x - \bar{x}|);
 	\end{equation}
 	\item b) v is a \textit{\textbf{(general) subgradient}} of $f$ at $\bar{x}$, written as $v \in \partial f(\bar{x})$, if there are sequences $x^{\nu} \rightarrow \bar{x}$, with $f(x^\nu) \rightarrow f(\bar{x})$, and
 	$v \in \hat{\partial} f(x^\nu)$ with $v^\nu \rightarrow v$.
 \end{itemize}
 \end{defi}
Note that notation $o$ in a), stands for a short-hand for an one-sided limit condition.
\begin{equation}\label{key}
\liminf_{x\rightarrow \bar{x}, x \neq \bar{x}}\dfrac{f(x) -f(\bar{x}) - \left<v,x-\bar{x}\right>}{|x - \bar{x}|} \ge 0
\end{equation}
Specifically, the generalized subgradient  has the following properties.
\begin{proper}[existence of the subgradient](Corollary 8.10 of \cite{var})
If the function $f: \mathbb{R}^n \rightarrow \bar{\mathbb{R}}$ is finite and lower semi-continuous at $\bar{x}$, then we have $\partial f(\bar{x})$ is not empty.
\end{proper}
\noindent With Property 1, we see that generalized subgradient always exists for finite and lower semi-continuous functions.
\begin{proper}[Generalized Fermat rule for local minimums](Theorem 10.1 of \cite{var})
	If a proper function $f: \mathbb{R}^n \rightarrow \bar{\mathbb{R}}$ has a local minimum at $\bar{x}$, then we have $0 \in \partial f(\bar{x})$.
\end{proper}
\noindent With Property 2, we define $\bar{x}$ as a \textbf{\textit{critical point}} of $f$, if $0 \in \partial{f}$.

 \subsection{The continuity of $\mathcal{R}_2(\boldsymbol{G})$}
 Property 2 of the generalized subdifferential shows that one could always find nonempty subdifferential for lower semi-continuous functions. For our objective function $\mathcal{F}(\cdot)$, it is easy to find that $\mathcal{L}$, $\mathcal{R}_1(\cdot)$, and $\mathcal{R}_3(\cdot)$ are continuous and thus must be lower semi-continuous. Now in this section, we provide a guarantee for nonempty generalized subdifferentials of $\mathcal{F}(\cdots)$ via proving Lemma 1,  i.e., the continuity (and thus lower semi-continuity) of $\mathcal{R}_2$. Note that $\mathcal{R}_2$ always depends on the predefined group number. In the following discussions, we will keep it fixed as $\kappa$. \\

\subsubsection{Proof of Lemma 1}
For the sake of simplicity, we just set $m = \min\{d,U\}$ throughout this proof. It is easy to find that $\mathcal{R}_2(\boldsymbol{G})= \sum_{i=\kappa+1}^{m}\sigma^2_i(\boldsymbol{G})$ reformulated as a composite $f \circ \sigma$. Here,  $\sigma= (\sigma_1(\boldsymbol{G}), \cdots, \sigma_{m}(\boldsymbol{G}))$ are the singular values of $\boldsymbol{G}$. $f: \mathbb{R}^{m }\rightarrow \bar{\mathbb{R}}$ is  the sum-of-squares of the last $m-\kappa$ elements with smallest absolute value.  In other words, denote a rearrangement of $x$ as $x^\downarrow$ with $|x^\downarrow_1| \ge |x^\downarrow_2|, \cdots \ge |x^\downarrow_{m}|$, then we have $f(x) = \sum_{i=\kappa+1}^{m} (x^\downarrow_i)^2$. Clearly, $\mathcal{R}_2(\cdot)$ is continuous if and only if $f(\cdot)$ is continuous. We only need to prove that $f(\cdot)$ is continuous.\\
\noindent For any given point $x_0 \in \mathbb{R}^m$, define $\delta_1, \delta_2$ as\footnote{The choice $\delta_1$ is to avoid the change of sign in the neighbor, whereas the choice of $\delta_2$  is to avoid the change of the order with respect to absolute values. } :
	\begin{align*}
	\delta_1 =\begin{cases} \min\limits_{
		m \ge i \ge 1, ~
		|x_{0,i}^\downarrow| > 0 }(|x_{0,i}^\downarrow|) &{if }~ \exists~ i,|x_{0,i}^\downarrow|  > 0 \\ 
	+\infty & otherwise \end{cases}, ~~
		\delta_2 = \begin{cases} \min\limits_{
			m \ge i > 1, ~
			|x_{0,i}^\downarrow| <|x_{0,i-1}^\downarrow|}(|x_{0,i-1}^\downarrow| - |x_{0,i}^\downarrow|) &{if }~ \exists~ i, i-1, |x_{0,i-1}^\downarrow| > |x_{0,i}^\downarrow|  \\ 
		+\infty & otherwise \end{cases}
	\end{align*}
Furthermore, define $M$ as :
\begin{align*}
	M = \begin{cases}
		1, & if ~~ \delta_1 = +\infty ~~~  and~~ \delta_2 =　+\infty \\
		\min\{\delta_1,\frac{\delta_2}{2}\} & otherwise
	\end{cases}
\end{align*}	
 Define an open ball \[\mathcal{B}(x_0,\delta) = \{x:  0<\norm{x - x_0} < \delta\}.\] With the choice of   $ 0<\delta < M$, we have $\forall x \in \mathcal{B}(x_0,\delta), m \ge i > \kappa$ :
 \[\left| |x_i^{\downarrow}|^2 - |x_{0,i}^{\downarrow}|^2\right| \le (|x_{0,i}^{\downarrow}|+\delta)^2 - (|x_{0,i}^{\downarrow}|)^2 \le (2|x_{0,i}^\downarrow|+\delta) \delta \]
 Then we have $\forall x \in \mathcal{B}(x_0,\delta)$ :
 \[|f(x) -f(x_0)| \le \sum_{i=\kappa+1}^m\left||x_i^{\downarrow}|^2 - |x_{0,i}^{\downarrow}|^2\right| \le \delta \sum_{i=\kappa+1}^m(2|x_{0,i}^\downarrow|+\delta) \le \delta(2\norm{x_0}_1 +mM) .  \]
 If $\epsilon > \delta(2\norm{x_0}_1 +mM) $, we have  $\delta < \frac{\epsilon}{ 2\norm{x_0}_1 +mM}.$ We then have $\forall \epsilon >0$, with the choice\footnote{actually it works for all $\delta < \min\{M,  \frac{\epsilon}{2\norm{x_0}_1 +mM}\} $} $\delta = \frac{1}{2}\min\{M,  \frac{\epsilon}{2\norm{x_0}_1 +mM}\}$, we have \[\forall x \in \mathcal{B}(x_0, \delta) , |f(x) - f(x_0)| < \epsilon.\] This immediately implies the continuity of $f(\cdot)$ at $x_0$.
\qed
 \subsection{Proof of Theorem 2}
 \textbf{proof of 1)}: Since ${\mathcal{L}}(\cdot)$ is Lipschitz continuous, for every iteration $k+1$, $\exists~$ a smallest number $\eta_{k+1}, and~ \gamma \ge \eta_{k+1} \ge0$, such that $\forall C \ge  \eta_{k+1} $ :
\begin{equation}\label{lp}
\begin{split}
\lk{k+1} \le& \lk{k} + \left<\nabla_{\boldsymbol{\theta}}\mathcal{L}, \Delta(\boldsymbol{\theta}^k)\right> +  \left<\nabla_{G}\mathcal{L}, \Delta(\boldsymbol{G}^k)\right> +  \left<\nabla_{P}\mathcal{L}, \Delta(\boldsymbol{P}^k)\right> \\
&+\frac{C}{2}\norm{\Delta(\boldsymbol{\theta}^k)}^2_2 + \frac{C}{2}\norm{\Delta(\boldsymbol{G}^k)}^2_F + \frac{C}{2}\norm{\Delta(\boldsymbol{P}^k)}^2_F
\end{split}
\end{equation}
The line search strategy guarantees that $\eta_{k+1} < \rho_{k+1}$ (via the $<$ in the $if$ condition of Algorithm 1).\\
For $k+1$-th iteration, the $\boldsymbol{\theta}$ subproblem:
\begin{equation}
\argmin_{\boldsymbol{\theta}} \dfrac{1}{2} \left\norm{\boldsymbol{\theta} -\tilde{\boldsymbol{\theta}^k} \right}_2^2 + \frac{\lambda_1}{\rho_k} \norm{\boldsymbol{\theta}}^2_2 \Leftrightarrow \argmin_{\boldsymbol{\theta}} \left<\nabla_{\boldsymbol{\theta}}\mathcal{L}, \boldsymbol{\theta}-\boldsymbol{\theta}^k)\right> + \frac{\rho_{k+1}}{2}\norm{\boldsymbol{\theta} -\boldsymbol{\theta}^k}^2_2 + \lambda_1 \norm{\boldsymbol{\theta}}^2_2
\end{equation}
is strongly convex. This implies that the solution $\boldsymbol{\theta}^{k+1}$ is the minimizer of this subproblem. In this sense, we have:
\begin{equation}\label{eq:con1}
\rth \le  \lambda_1 \norm{\boldsymbol{\theta}^{k}}^2_2
\end{equation}
Similarly, we have the $\boldsymbol{P}$ subproblem is strongly-convex, and :
\begin{equation}\label{eq:con2}
\rp \le \lambda_3\norm{\boldsymbol{P}^{k}}_{1,2}
\end{equation}
In Proposition 2, we proved the optimality of the $\boldsymbol{G}$ problem, thus we have, for the $k+1$-th iteration :
\begin{equation}\label{eq:con3}
\rgg \le \lambda_2 \regg{k}
\end{equation}
Substituting $C= \eta_{k+1}$ into (\ref{lp}) and summing up (\ref{lp}), (\ref{eq:con1})-(\ref{eq:con3}), we have:
\begin{equation}\label{eq:dec}
\fk{k+1} \le \fk{k} -  \frac{\rho_{k+1}-\eta_{k+1}}{2} 
\left( \Delta(\boldsymbol{\theta}^k)+ \Delta(\boldsymbol{G}^k) + \Delta(\boldsymbol{P}^k)  \right)
\end{equation}
 Then by choosing $C_{k+1} =\frac{\rho_{k+1}-\eta_{k+1}}{2} $, we complete the proof for 1).
\\\\
\textbf{proof of 2)}: By adding up (\ref{eq:dec}) for $k=1,2,\cdots$, and the fact that $\mathcal{F}(\cdot,\cdot,\cdot) \ge 0$, we have
\begin{equation}\label{eq:finite}
\sum_{k=1}^\infty \frac{\rho_k - \eta_k}{2}\left(\Delta(\boldsymbol{\theta}^k)+ \Delta(\boldsymbol{G}^k) + \Delta(\boldsymbol{P}^k) \right) \le \fk{0}
\end{equation}
Then $\fk{0} < \infty$, immediately implies that :
\begin{equation}
\Delta(\boldsymbol{\theta}^k) \rightarrow 0, ~ \Delta(\boldsymbol{G}^k)  \rightarrow 0, \Delta(\boldsymbol{P})^k  \rightarrow 0
\end{equation} 
This completes the proof of 2).\\ \\
\textbf{proof of 3)} According to the decent property proved in 1), we have, for all $k\ge 1$: $\fk{k} \le  {\mathcal{F}(\boldsymbol{\theta}^{0}, \boldsymbol{G}^{0}, \boldsymbol{P}^{0})}$. This immediate implies that $\forall k \ge 1, \norm{\boldsymbol{\theta}^k}_2, \norm{\boldsymbol{P}^k}_{1,2}$ is bounded, thus $\boldsymbol{\theta}^k$ and  $\boldsymbol{G}^k$ are bounded.\\ Next, we prove the boundedness of $\{\boldsymbol{G}^k\}_k$ by induction.
\begin{itemize}
	\item For the basic case, we know $\boldsymbol{G}^0$ is bounded by assumption. 
	\item If $\boldsymbol{G}^k$ is bounded, note that  $\boldsymbol{G}^{ref_{k+1}} = \boldsymbol{G}^k$. Together with the fact that $\boldsymbol{\theta}^k$ and $\boldsymbol{P}^k$ are bounded, we have reached that $\boldsymbol{\tilde{G}}^{k+1}$ is bounded. By Solution in Eq.(\ref{solveP}), we see that the only difference between  $\boldsymbol{{G}}^{k+1}$ and $\boldsymbol{\tilde{G}}^{k+1}$ is that $\boldsymbol{{G}}^{k+1}$ shrinks the last $m - \kappa$ singular values of  $\boldsymbol{\tilde{G}}^{k+1}$. Obviously we attain that $\boldsymbol{{G}}^{k+1}$ is bounded. 
\end{itemize}  
Then we end the proof of the induction.
\\
\\
\textbf{Proof of 4)} Given a subsequence $\{\boldsymbol{\theta}^{k_j},\boldsymbol{G}^{k_j}, \boldsymbol{P}^{k_j}\}_j$  , for the $k_j+1$-th iteration, according to necessary condition of optimality of all three subproblems Eq.(\ref{Ptheta})-Eq.(\ref{Pp}),  we have :
\begin{equation}\label{eq:sub1}
0 = \nabla_{\theta}\mathcal{L}(\boldsymbol{\theta}^{k_j},\boldsymbol{G}^{k_j}, \boldsymbol{P}^{k_j}) + \rho_{k_j+1}(\boldsymbol{\theta}^{k_j+1} - \boldsymbol{\theta}^{k_j}) + \lambda_1 \nabla \mathcal{R}_1(\theta^{k_j+1})
\end{equation}
\begin{equation}\label{eq:sub2}
0 \in \nabla_{\boldsymbol{G}}\mathcal{L}(\boldsymbol{\theta}^{k_j},\boldsymbol{G}^{k_j}, \boldsymbol{P}^{k_j}) + \rho_{k_j+1}(\boldsymbol{G}^{k_j+1} - \boldsymbol{G}^{k_j}) + \lambda_2\partial(\mathcal{R}_2(\boldsymbol{G}^{k_j+1}))
\end{equation}
\begin{equation}\label{eq:sub3}
0 \in \nabla_{\boldsymbol{P}}\mathcal{L}(\boldsymbol{\theta}^{k_j},\boldsymbol{G}^{k_j}, \boldsymbol{P}^{k_j}) + \rho_{k_j+1}(\boldsymbol{P}^{k_j+1} - \boldsymbol{P}^{k_j}) + \lambda_3\partial(\mathcal{R}_3(\boldsymbol{P}^{k_j+1}))
\end{equation}
We assume that this subsequence has a limit point $\boldsymbol{
\theta}^*, \boldsymbol{G}^*, \boldsymbol{P}^*$,i.e, when $j \rightarrow \infty $, $\boldsymbol{\theta}^{k_j} \rightarrow \boldsymbol{\theta}^*$, $\boldsymbol{G}^{k_j} \rightarrow \boldsymbol{G}^*$, $\boldsymbol{P}^{k_j} \rightarrow \boldsymbol{P}^*$. Then according to the continuity of $\nabla \mathcal{L}$, the conclusion of 2), and the outer semi-continuity of the limiting sub-differential $\partial(\cdot)$, we have :
\begin{equation}\label{eq:fin1}
0 \in \nabla_{\boldsymbol{\theta}}\mathcal{L}(\boldsymbol{\theta}^{*},\boldsymbol{G}^{*}, \boldsymbol{P}^{*})  + \lambda_1\nabla(\mathcal{R}_1(\boldsymbol{\theta}^{*}))
\end{equation}
\begin{equation}\label{eq:fin2}
0 \in \nabla_{\boldsymbol{G}}\mathcal{L}(\boldsymbol{\theta}^{*},\boldsymbol{G}^{*}, \boldsymbol{P}^{*})  + \lambda_2\partial(\mathcal{R}_2(\boldsymbol{G}^{*}))
\end{equation}
\begin{equation}\label{eq:fin3}
0 \in \nabla_{\boldsymbol{P}}\mathcal{L}(\boldsymbol{\theta}^{*},\boldsymbol{G}^{*}, \boldsymbol{P}^{*})  + \lambda_3\partial(\mathcal{R}_3(\boldsymbol{P}^{*}))
\end{equation}
This implies that
\begin{equation}
\begin{split}
0 \in & \begin{pmatrix}\nabla_{\boldsymbol{G}}\mathcal{L}(\boldsymbol{\theta}^{*},\boldsymbol{G}^{*}, \boldsymbol{P}^{*}) \\\nabla_{\boldsymbol{P}}\mathcal{L}(\boldsymbol{\theta}^{*},\boldsymbol{G}^{*}, \boldsymbol{P}^{*})\\
\nabla_{\boldsymbol{G}}\mathcal{L}(\boldsymbol{\theta}^{*},\boldsymbol{G}^{*}, \boldsymbol{P}^{*}) 
\end{pmatrix}  + \{\lambda_1\nabla(\mathcal{R}_1(\boldsymbol{\theta}^{*}))\} \otimes \lambda_2\partial(\mathcal{R}_2(\boldsymbol{G}^{*})) \otimes  \lambda_3\partial(\mathcal{R}_3(\boldsymbol{P}^{*}))\\
 = & ~~
\nabla\mathcal{L}(\boldsymbol{\theta}^{*},\boldsymbol{G}^{*}, \boldsymbol{P}^{*}) + \partial(\lambda_1\mathcal{R}_1(\boldsymbol{\theta}^{*}) + \lambda_2\mathcal{R}_2(\boldsymbol{G}^{*}) + \lambda_3\mathcal{R}_3(\boldsymbol{P}^{*}))\\
=& ~ ~ \partial(\fk{*}  )
\end{split}
\end{equation}
where $\otimes$ denotes the Cartesian product for sets, the first equality is due to the sub-differential calculus for separable functions \cite{var}, and the second equality is due to the sum rule. Obviously the proof is complete, since the subsequence is arbitrarily chosen.
\\
\\
\textbf{Proof of 5)}
According to (\ref{eq:finite}), we have, for any $T \ge 1$:
\begin{equation}
\min_{1\le k < T}\{(\frac{\rho_k - \eta_k}{2})\}\sum_{k =0}^{T-1}\norm{\boldsymbol{\theta}^{k+1} - \boldsymbol{\theta}^k}^2_2  \le \sum_{k =0}^{T-1}\frac{\rho_k - \eta_k}{2}\norm{\boldsymbol{\theta}^{k+1} - \boldsymbol{\theta}^k}^2_2   \le   \fk{0}
\end{equation}
\begin{equation}
\min_{1\le k < T}\{(\frac{\rho_k - \eta_k}{2})\}\sum_{k =0}^{T-1}\norm{\boldsymbol{G}^{k+1} - \boldsymbol{G}^k}^2_F  \le \sum_{k =0}^{T-1}\frac{\rho_k - \eta_k}{2}\norm{\boldsymbol{G}^{k+1} - \boldsymbol{G}^k}^2_F   \le   \fk{0}
\end{equation}
\begin{equation}
\min_{1\le k < T}\{(\frac{\rho_k - \eta_k}{2})\}\sum_{k =0}^{T-1}\norm{\boldsymbol{P}^{k+1} - \boldsymbol{P}^k}^2_F  \le \sum_{k =0}^{T-1}\frac{\rho_k - \eta_k}{2}\norm{\boldsymbol{P}^{k+1} - \boldsymbol{P}^k}^2_F  \le   \fk{0}
\end{equation}
Above all, the proof is complete by choosing $C_T = \dfrac{ \fk{0}  }{\min_{1\le k \le T}\{(\dfrac{\rho_k - \eta_k}{2})\}}$.
\qed

\section{Generalization Analysis}

\subsubsection{Notations}For the loss function, we denote $\phi(x) = (1-x)^2$, and $l(f,x_1,x_2) = (1-(f(x_1)-f(x_2)))^2$. The proofs are based on two distinct samples $\mathcal{S}$ and $\mathcal{S}'$, where $\mathcal{S}_{+,i}$  and $\mathcal{S}'_{+,i}$ are sampled from $\mathcal{D}_{+,i}$, $\mathcal{S}_{-,i}$  and $\mathcal{S}'_{-,i}$ are sampled from $\mathcal{D}_{-,i}$. The empirical loss on $\mathcal{S}$ and $\mathcal{S}'$ are denoted as $\mathcal{L}(\boldsymbol{W})$ and $\mathcal{L}'(\boldsymbol{W})$. For the $i$-th user, the empirical losses are denoted as 

\[\mathcal{L}^{(i)}(\boldsymbol{W}^{(i)}) = \frac{1}{n_{+,i}n_{-,i}}\sum _{x_p \in S_{+,i}}\sum_{\boldsymbol{x}_q \in \mathcal{S}_{-,i}}l(f^{(i)},\boldsymbol{x}_p,\boldsymbol{x}_q), ~~ \mathcal{L}^{(i)'}(\boldsymbol{W}^{(i)}) = \frac{1}{n_{+,i}n_{-,i}}\sum _{x_p \in S'_{+,i}}\sum_{\boldsymbol{x}_q \in \mathcal{S}'_{-,i}}l(f^{(i)},\boldsymbol{x}_p,\boldsymbol{x}_q),\] respectively.

We adopt the generalized Rademacher averages proposed in \cite{genauc} and further generate it to deal with the multi-task case. For the $i$th user the generalized Rademacher average on a finite sample $\mathcal{S}$ is given as :
\[\hat{\mathcal{E}}^{(i)} = 4\mathbb{E}_{\sigma^{(i)},\nu^{(i)}}\sup_{ \Theta}\frac{1}{n_{+,i}n_{-,i}}\sum_{\boldsymbol{x}_p \in \mathcal{S}_{+,i}}\sum_{\boldsymbol{x}_q \in \mathcal{S}_{-,i}}\frac{\sigma^{(i)}_{p}+\nu^{(i)}_{q}}{2}l(f^{(i)},\boldsymbol{x}_1,\boldsymbol{x}_2),\]
where $\boldsymbol{\sigma^{(i)}} =[\sigma^{(i)}_1 \cdots \sigma^{(i)}_{n_{+,i}}], \boldsymbol{\nu}^{(i)} =[\nu^{(i)}_1 \cdots \nu^{(i)}_{n_{-,i}}] $ are two vectors of independent Rademacher random variables.
Denote $\mathcal{S}_i =\{\mathcal{S}_{+,i},\mathcal{S}_{-,i}\}$, $\mathcal{E}^{(i)}$ is the corresponding expectation over $\mathcal{S}_i$ :$ \mathcal{E}^{(i)} = \mathbb{E}_{\mathcal{S}_i \sim \mathcal{D}_{+,i} \otimes \mathcal{D}_{-,i}} \hat{\mathcal{E}}^{(i)}$. For all users, we denote $\hat{\mathcal{E}}= \sum_i \hat{\mathcal{E}^{(i)}}$ and ${\mathcal{E}}=  \mathbb{E}_{\mathcal{S} \sim \mathcal{D}}\hat{\mathcal{E}} = \sum_i{\mathcal{E}^{(i)}}.$

\begin{lem}[Mcdiarmid]\cite{mcbound}
	Let $X_1,\cdots,X_m$ be independent random variables all taking values in the set $\mathcal{X}$. Let $f: \mathcal{X} \rightarrow \mathbb{R}$ be a function of
	$X_1,\cdots,X_m$ that satisfies:
	\[\sup_{\boldsymbol{x},\boldsymbol{x}'}|f(x_1,\cdots,x_i,\cdots, x_m) - f(x_1, \cdots,x'_i \cdots,x_,) | \le c_i,\]
	with $\boldsymbol{x} \neq \boldsymbol{x}'$.
 Then for all $\epsilon >0$,
	\[\mathbb{P}[ \mathbb{E}(f) - f \ge \epsilon ] \le exp\left(\dfrac{-2\epsilon^2}{\sum_{i=1}^mc_i^2} \right) \]
\end{lem}

\begin{lem}\label{l1} The following facts hold:
	\begin{itemize}
		\item	a) For $i=1,2,\cdots,U$,and for a function $f$ with 
		$\norm{f(x)}_\infty \leq B$, we have for any 
		$\boldsymbol{x}_1,\boldsymbol{x}'_1,\boldsymbol{x}_2,\boldsymbol{x}'_2 
		$, the following two inequalities 
		hold 
		\begin{equation}\label{ine1}
		|l(f,\boldsymbol{x}_1,\boldsymbol{x}_2) - 
		l(f,\boldsymbol{x}_1,\boldsymbol{x}'_2)| \leq 2(2+2B)B
		\end{equation}
		\begin{equation}\label{ine2}
		|l(f,\boldsymbol{x}_1,\boldsymbol{x}_2) - 
		l(f,\boldsymbol{x}'_1,\boldsymbol{x}_2)| \leq 2(2+2B)B
		\end{equation}.
		\item b) If $x \leq B$ always holds, then we have that $\phi$ is $2(1+B)$-Lipschitz continuous, i.e. :
		\[|\phi(x) - \phi(x')| \le 2(1+B) |x -x'|, ~\forall x \leq B,x' \leq B, x \ne x'.\] 
	\end{itemize}

\end{lem}
\begin{proof}	
	\begin{equation*}\label{key}
	\begin{split}
	&|l(f,\boldsymbol{x}_1,\boldsymbol{x}_2) - 
	l(f,\boldsymbol{x}_1,\boldsymbol{x}'_2)| \leq \\
	&|(1-(f(\boldsymbol{x}_1)-f(\boldsymbol{x}_2)))^2-
	(1-(f(\boldsymbol{x}_1)-f(\boldsymbol{x}'_2)))^2| \leq \\
	& 
	|2-(f(\boldsymbol{x}_1)-f(\boldsymbol{x}_2)+f(\boldsymbol{x}_1)-f_2(\boldsymbol{x}'_2))||f(\boldsymbol{x}'_2)-f(\boldsymbol{x}_2)|
	\leq \\
	& (2+2B)2B = 2(2+2B)B
	\end{split}
	\end{equation*}
	Eq.(\ref{ine2}) could be proved directly following the same routine.\\
	For b), we have:
	\[|\phi(x) - \phi(x')| \le  |2-(x+x') | |x -x'| \le 2(1+B)|x-x'| \]
\end{proof}
\begin{lem}Under the assumption of Theorem 1, with probability at least $1- \delta$,  for any $\mathcal{S}$ sampled from $\mathcal{D}$, we have:
\[\sup_{\Theta} \big(\mathbb{E}_{\mathcal{S} \sim \mathcal{D}}(\mathcal{L}(\boldsymbol{W})) - {\mathcal{L}}(\boldsymbol{W})\big) \le \mathcal{E} +2\sqrt{2} (1+\zeta)\zeta\sqrt{\frac{ln(\frac{1}{\delta})}{\sum_{i=1}^Um_i\chi_i(1-\chi_i)}}. \]
\end{lem}
\begin{proof}
Denote $\mathcal{S}=\{S_{+,i},S_{-,i}\}_{i=1}^U$ as a dataset for all users $i=1,\cdots, u$. For a given user $i$, the positive labeled dataset is denoted as $S_{+,i}=\{\boldsymbol{x}^{(+,i)}_l\}_{l=1}^{n_i}$, whereas the negative labeled dataset is denoted as  $S_{-,i}=\{\boldsymbol{x}^{(-,i)}_k\}_{k=1}^{n_i}$. Let $\mathcal{S}'$ be another dataset with only one instance different from $\mathcal{S}$. We denote the instance pair that lead to this different as $\boldsymbol{x}$ and $\boldsymbol{x}'$, where $\boldsymbol{x} \in \mathcal{S}$ and $\boldsymbol{x}' \in \mathcal{S}$.  Furthermore, we denote the corresponding empirical risk for  $\mathcal{S}$ and $\mathcal{S}'$ as $\mathcal{L}(\boldsymbol{W})$ and $\mathcal{L}'(\boldsymbol{W})$, respectively. 
\\

Given $\Theta$ defined in the main paper, and the assumption on the input instances, we have, for the $i$-th user and for any input $\boldsymbol{x}$ with $\norm{\boldsymbol{x}} \le \Delta_{\chi}$ :
\begin{equation*}
|f^{(i)}(\boldsymbol{x})| = |\boldsymbol{x}^{\top}(\boldsymbol{\theta} + \boldsymbol{G}^{(i)} + \boldsymbol{P}^{(i)})| \leq \norm{\boldsymbol{x}}\cdot\norm{\boldsymbol{\theta}+\boldsymbol{G}^{(i)}+\boldsymbol{P}^{(i)}} \leq \Delta_\chi(\psi_1 +\sqrt{\psi_2 + \kappa \cdot \sigma_{max}^2 } + \psi_3) \triangleq \zeta.
\end{equation*}
If $\boldsymbol{x} \in S_{+,i}$ and $\boldsymbol{x}' \in \mathcal{S}'_{+,i}$, for all choices of such $\mathcal{S}$ and $\mathcal{S}'$,  we have : 
\begin{equation*}\label{key}
\begin{split}
|(\mathbb{E}_{\mathcal{D}}(\mathcal{L}(\boldsymbol{W}))- \mathcal{L}(\boldsymbol{W}))-(\mathbb{E}_{\mathcal{D}}(\mathcal{L}'(\boldsymbol{W}))-\mathcal{L}'(\boldsymbol{W}))| &= |\mathcal{L}(\boldsymbol{W})-\mathcal{L}'(\boldsymbol{W})| \\    &\leq \dfrac{1}{n_{+,i}n_{-,i}}\sum\limits_{\boldsymbol{x}_q \in \mathcal{S}_{-,i}} 
|l(f^{(i)},\boldsymbol{x},\boldsymbol{x}_q)-l(f^{(i)},\boldsymbol{x}',\boldsymbol{x}_q)|\\
& \ \  \leq \dfrac{4}{n_{+,i}} (1+\zeta)\zeta, 
\end{split}
\end{equation*}
where the last inequality is from Lemma 3-(a). Since the choice of $(\boldsymbol{\theta},\boldsymbol{G},\boldsymbol{P})$ is arbitrary, we have :
\[\sup_{\boldsymbol{x},\boldsymbol{x}'}|\sup_{{\Theta}}(\mathbb{E}_{\mathcal{D}}(\mathcal{L}(\boldsymbol{W}))- \mathcal{L}(\boldsymbol{W}))-\sup_{\Theta}(\mathbb{E}_{\mathcal{D}}(\mathcal{L}'(\boldsymbol{W}))-\mathcal{L}'(\boldsymbol{W}))|\le \sup_{\boldsymbol{x},\boldsymbol{x}',\Theta} |(\mathbb{E}_{\mathcal{D}}(\mathcal{L}(\boldsymbol{W}))- \mathcal{L}(\boldsymbol{W}))-(\mathbb{E}_{\mathcal{D}}(\mathcal{L}'(\boldsymbol{W}))-\mathcal{L}'(\boldsymbol{W}))| \le \dfrac{4}{n_{+,i}} (1+\zeta)\zeta  \]
 \\
Similarly, if $x \in S_{-,i}$, $x' \in S'_{-,i}$ we have
\begin{equation*}\label{key}
\begin{split}
\sup_{\boldsymbol{x},\boldsymbol{x}'}|\sup_{{\Theta}}(\mathbb{E}_{\mathcal{D}}(\mathcal{L}(\boldsymbol{W}))- \mathcal{L}(\boldsymbol{W}))-\sup_{\Theta}(\mathbb{E}_{\mathcal{D}}(\mathcal{L}'(\boldsymbol{W}))-\mathcal{L}'(\boldsymbol{W}))|  & \leq \dfrac{4}{n_{-,i}} (1+\zeta)\zeta 
\end{split}
\end{equation*}
According to Lemma 2,  We 
have, $\forall \epsilon > 0$
\begin{equation*}\label{key}
\begin{split}
& 
\mathbb{P}_{\mathcal{D}}\Bigg\{\sup_{{\Theta}}(\mathbb{E}_{\mathcal{D}}(\mathcal{L}(\boldsymbol{W}))- \mathcal{L}(\boldsymbol{W}))- \mathbb{E}_{\mathcal{D}}\left[\sup_{{\Theta}}(\mathbb{E}_{\mathcal{D}}(\mathcal{L}(\boldsymbol{W}))- \mathcal{L}(\boldsymbol{W}))\right]
\ge \epsilon\Bigg\}  \\
& \leq  exp\left\{\dfrac{-2\epsilon^2}{16(1+\zeta)^2\zeta^2 
	\left(\sum\limits_{i=1
	}^U\dfrac{n_{+,i}}{n_{+,i}^2}+\dfrac{n_{-,i}}{n_{-,i}^2}\right)}\right\}\\
& \leq exp\Bigg\{\dfrac{-\epsilon^2 \sum_{i=1}^u\chi_i(1-\chi_i) n_i}{8(1+\zeta)^2(\zeta)^2}\Bigg\}
\end{split}
\end{equation*}
Finally, by substituting $ exp\Bigg\{\dfrac{-\epsilon^2 \sum_{i=1}^u\chi_i(1-\chi_i) n_i}{8(1+\zeta)^2(\zeta)^2}\Bigg\}$ with $\delta$
, we have :
\[\sup_{{\Theta}}(\mathbb{E}_{\mathcal{D}}(\mathcal{L}(\boldsymbol{W}))- \mathcal{L}(\boldsymbol{W}))  \le \mathbb{E}_{\mathcal{D}}\left[\sup_{{\Theta}}(\mathbb{E}_{\mathcal{D}}(\mathcal{L}(\boldsymbol{W}))- \mathcal{L}(\boldsymbol{W}))\right] + 2\sqrt{2} (1+\zeta)\zeta\sqrt{\frac{ln(\frac{1}{\delta})}{\sum_{i=1}^Um_i\chi_i(1-\chi_i)}}, \]
with probability at least $1 - \delta$.
Given two different samples $\mathcal{S}$, $\tilde{\mathcal{S}}$ drawn from $\mathcal{D}$, denote their empirical loss as $\mathcal{L}(\boldsymbol{W})$, and $\mathcal{L}'(\boldsymbol{W})$, respectively. Then, we have :
\begin{equation}
\begin{split}
&~~\mathbb{E}_{\mathcal{D}}\left[\sup_{{\Theta}}(\mathbb{E}_{\mathcal{D}}(\mathcal{L}(\boldsymbol{W}))- \mathcal{L}(\boldsymbol{W}))\right] \\
\overset{(1)}{\le} &~~ \mathbb{E}_{\mathcal{S} \sim \mathcal{D}, \tilde{\mathcal{S}} \sim \mathcal{D}}\left[\sup_{\Theta}(\mathcal{L}'(\boldsymbol{W}) -\mathcal{L}(\boldsymbol{W}) )\right]\\
\le  & ~~  \mathbb{E}_{\mathcal{S} \sim \mathcal{D}, \tilde{\mathcal{S}} \sim \mathcal{D}}\left[\sum_{i=1}^U\sup_{\Theta}(\mathcal{L}^{(i)'}(\boldsymbol{W}^{(i)}) -\mathcal{L}^{(i)}(\boldsymbol{W}^{(i)}) )\right]\\
= & ~~ \sum_{i=1}^U\mathbb{E}_{\mathcal{S}_i \sim \mathcal{D}_i, \tilde{\mathcal{S}}_i \sim \mathcal{D}_i}\left[\sup_{\Theta}(\mathcal{L}^{(i)'}(\boldsymbol{W}^{(i)}) -\mathcal{L}^{(i)}(\boldsymbol{W}^{(i)}) )\right]\\
\overset{(2)}{\le} & ~~ \mathcal{E},
\end{split}
\end{equation}
where (1) and (2) follows the proof of Theorem 1 in \cite{genauc}. We then complete the proof.
\end{proof}
\begin{lem}
Under the assumption of Theorem 3, we have, $\forall \delta \in (0,1)$:
\[\mathcal{E} \le \hat{\mathcal{E}} + 8\sqrt{2} (1+\zeta)\zeta\sqrt{\frac{ln(\frac{1}{\delta})}{\sum_{i=1}^Um_i\chi_i(1-\chi_i)}} \]
holds with probability at least $1-\delta$.
\end{lem}
The proof is similar as Lemma 4, which is omitted it here.
\begin{lem}
Under the assumption of Theorem 3, we have:
\[\hat{\mathcal{E}} \le \sum_{i=1}^U\frac{8\sqrt{2}C\Delta_{\chi}(1+\zeta)}{\sqrt{(n_i\chi_i(1-\chi_i))}},\]
where $C=(\psi_1 +\sqrt{\psi_2 + \kappa \cdot \sigma_{max}^2 } + \psi_3)$.
\end{lem}
\begin{proof}
\begin{equation}
\begin{split}
\hat{\mathcal{E}}^{(i)}\le &~~ \frac{4\sqrt{2}C(1+\zeta)\sqrt{(n_{+,i}+n_{-,i})}}{n_{+,i}n_{-,i}}\sqrt{\sum_{x_p\in\mathcal{S}_{+,i}}\sum_{x_q\in\mathcal{S}_{-,i}}\norm{\boldsymbol{x}_p-\boldsymbol{x}_q}^2}\\
\le &~~ \frac{8\sqrt{2}C\Delta_{\chi}(1+\zeta)}{\sqrt{(n_i\chi_i(1-\chi_i))}},
\end{split}
\end{equation}
where the first inequality follows Lemma 5 of \cite{genauc}, and  Lemma 3 b) in our paper. Then according to the definition of $\hat{\mathcal{E}}$, we have: 
\[\hat{\mathcal{E}} \le \sum_{i=1}^U\frac{8\sqrt{2}C\Delta_{\chi}(1+\zeta)}{\sqrt{(n_i\chi_i(1-\chi_i))}}\]
\end{proof}
\subsubsection{Proof of Theorem 3}
The proof follows Lemma 4-6 , the union bound, and the fact that the fact that $\mathbb{E}_\mathcal{D}\big(\sum_i\ell^{i}_{AUC}(\boldsymbol{W}^{(i)})\big) \le \mathbb{E}_\mathcal{D}(\mathcal{L}(\boldsymbol{W})\big).$
\qed

\end{document}